%% file: main.tex
\documentclass[nohyperref]{article}

\usepackage{microtype}
\usepackage{graphicx}
\usepackage{booktabs} %
\usepackage{algorithm}
\usepackage[noend]{algorithmic}

\usepackage[
colorlinks=true,
linkcolor=blue,
  citecolor=blue,
  filecolor=blue,
  urlcolor=black,
backref=page]{hyperref}

\usepackage[final]{include/automl}

\usepackage[sort&compress]{natbib}
\usepackage{amsmath}

\usepackage[english]{babel} %

\usepackage{mathtools}

\usepackage{amsthm}
\newcommand{\norm}[1]{\left\lVert#1\right\rVert}

\newcommand{\vbar}{\,|\,}

\usepackage[capitalize,noabbrev]{cleveref}

\newtheorem{theorem}{Theorem}[section]

\newtheorem{assumption}[theorem]{Assumption}

\usepackage[textsize=tiny]{todonotes}
\usepackage{wrapfig}

\usepackage{adjustbox} 

\usepackage[xindy,acronym,toc,smallcaps,nomain]{glossaries}
\newacronym{BO}{bo}{Bayesian optimization}
\newacronym{GP}{gp}{Gaussian Process}
\newacronym{TR}{tr}{Trust Regions}
\newacronym{PB2}{pb2}{Population Based Bandit}
\newacronym{RL}{rl}{Reinforcement Learning}
\newacronym{PBT}{pbt}{Population-based Training}
\newacronym{BG-PBT}{bg-pbt}{Bayesian Generational Population-based Training}
\newacronym{PPO}{ppo}{Proximal Policy Optimization}
\newacronym{UCB}{ucb}{Upper Confidence Bound}
\newacronym{RS}{rs}{Random Search}
\newacronym{MDP}{mdp}{Markov Decision Process}
\newacronym{MLP}{mlp}{Multi-layer Perceptron}
\newacronym{AutoRL}{AutoRL}{Automated Reinforcement Learning}
\newacronym{NAS}{nas}{Neural Architecture Search}
\glsdisablehyper

\newcommand{\our}{\gls{BG-PBT}\xspace}        %
\newcommand{\ourNAS}{\gls{BG-PBT}\xspace}      %

\title{Bayesian Generational Population-Based Training}

\author[1]{\nameemail{Xingchen Wan}{xwan@robots.ox.ac.uk}}
\author[1]{\nameemail{Cong Lu}{conglu@robots.ox.ac.uk}}
\author[1]{\nameemail{Jack Parker-Holder}{jackph@robots.ox.ac.uk}}
\author[1]{\nameemail{Philip J. Ball}{ball@robots.ox.ac.uk}}
\author[2]{\nameemail{Vu Nguyen}{vu@ieee.org}}
\author[1]{\nameemail{Binxin Ru}{robin@robots.ox.ac.uk}}
\author[1]{\nameemail{Michael A. Osborne}{mosb@robots.ox.ac.uk}}

\affil[1]{Machine Learning Research Group, University of Oxford, Oxford, UK}
\affil[2]{Amazon, Adelaide, Australia}

\hypersetup{%
  pdfauthor={}, %
  pdftitle={},
  pdfsubject={},
  pdfkeywords={}
}

\begin{document}
\input{macros.tex}

\maketitle

\input{01_abstract}

\input{02_intro}
\input{03_background}

\input{05_method}
\input{06_experiments}

\input{07_conclusion}

\input{08_limitations}
\input{10_ack}

\bibliography{references}
\bibliographystyle{apalike}

\newpage
\appendix

\input{11_appendix}

\end{document}

%% file: macros.tex
\global\long\def\se{\hat{\text{se}}}%

\global\long\def\interior{\text{int}}%

\global\long\def\boundary{\text{bd}}%

\global\long\def\new{\text{*}}%

\global\long\def\stir{\text{Stirl}}%

\global\long\def\dist{d}%

\global\long\def\HX{\entro\left(X\right)}%
 
\global\long\def\entropyX{\HX}%

\global\long\def\HY{\entro\left(Y\right)}%
 
\global\long\def\entropyY{\HY}%

\global\long\def\HXY{\entro\left(X,Y\right)}%
 
\global\long\def\entropyXY{\HXY}%

\global\long\def\mutualXY{\mutual\left(X;Y\right)}%
 
\global\long\def\mutinfoXY{\mutualXY}%

\global\long\def\xnew{y}%

\global\long\def\bx{\mathbf{x}}%

\global\long\def\bh{\mathbf{h}}%

\global\long\def\bw{\mathbf{w}}%

\global\long\def\bz{\mathbf{z}}%

\global\long\def\bu{\mathbf{u}}%

\global\long\def\bs{\boldsymbol{s}}%

\global\long\def\bk{\mathbf{k}}%

\global\long\def\bX{\mathbf{X}}%

\global\long\def\tbx{\tilde{\bx}}%

\global\long\def\by{\mathbf{y}}%

\global\long\def\bY{\mathbf{Y}}%

\global\long\def\bZ{\boldsymbol{Z}}%

\global\long\def\bU{\boldsymbol{U}}%

\global\long\def\bv{\boldsymbol{v}}%

\global\long\def\balpha{\boldsymbol{\alpha}}

\global\long\def\btheta{\boldsymbol{\theta}}

\global\long\def\bn{\boldsymbol{n}}%

\global\long\def\bV{\boldsymbol{V}}%

\global\long\def\bK{\boldsymbol{K}}%

\global\long\def\bbeta{\gvt{\beta}}%

\global\long\def\bmu{\gvt{\mu}}%

\global\long\def\btheta{\boldsymbol{\theta}}%

\global\long\def\blambda{\boldsymbol{\lambda}}%

\global\long\def\bgamma{\boldsymbol{\gamma}}%

\global\long\def\bpsi{\boldsymbol{\psi}}%

\global\long\def\bphi{\boldsymbol{\phi}}%

\global\long\def\bpi{\boldsymbol{\pi}}%

\global\long\def\eeta{\boldsymbol{\eta}}%

\global\long\def\bomega{\boldsymbol{\omega}}%

\global\long\def\bepsilon{\boldsymbol{\epsilon}}%

\global\long\def\btau{\boldsymbol{\tau}}%

\global\long\def\bSigma{\gvt{\Sigma}}%

\global\long\def\realset{\mathbb{R}}%

\global\long\def\realn{\realset^{n}}%

\global\long\def\integerset{\mathbb{Z}}%

\global\long\def\natset{\integerset}%

\global\long\def\integer{\integerset}%

\global\long\def\natn{\natset^{n}}%

\global\long\def\rational{\mathbb{Q}}%

\global\long\def\rationaln{\rational^{n}}%

\global\long\def\complexset{\mathbb{C}}%

\global\long\def\comp{\complexset}%

\global\long\def\compl#1{#1^{\text{c}}}%

\global\long\def\and{\cap}%

\global\long\def\compn{\comp^{n}}%

\global\long\def\comb#1#2{\left({#1\atop #2}\right) }%

\global\long\def\nchoosek#1#2{\left({#1\atop #2}\right)}%

\global\long\def\param{\vt w}%

\global\long\def\Param{\Theta}%

\global\long\def\meanparam{\gvt{\mu}}%

\global\long\def\Meanparam{\mathcal{M}}%

\global\long\def\meanmap{\mathbf{m}}%

\global\long\def\logpart{A}%

\global\long\def\simplex{\Delta}%

\global\long\def\simplexn{\simplex^{n}}%

\global\long\def\dirproc{\text{DP}}%

\global\long\def\ggproc{\text{GG}}%

\global\long\def\DP{\text{DP}}%

\global\long\def\ndp{\text{nDP}}%

\global\long\def\hdp{\text{HDP}}%

\global\long\def\gempdf{\text{GEM}}%

\global\long\def\ei{\text{EI}}%

\global\long\def\rfs{\text{RFS}}%

\global\long\def\bernrfs{\text{BernoulliRFS}}%

\global\long\def\poissrfs{\text{PoissonRFS}}%

\global\long\def\grad{\gradient}%
 
\global\long\def\gradient{\nabla}%

\global\long\def\cpr#1#2{\Pr\left(#1\ |\ #2\right)}%

\global\long\def\var{\text{Var}}%

\global\long\def\Var#1{\text{Var}\left[#1\right]}%

\global\long\def\cov{\text{Cov}}%

\global\long\def\Cov#1{\cov\left[ #1 \right]}%

\global\long\def\COV#1#2{\underset{#2}{\cov}\left[ #1 \right]}%

\global\long\def\corr{\text{Corr}}%

\global\long\def\sst{\text{T}}%

\global\long\def\SST{\sst}%

\global\long\def\ess{\mathbb{E}}%

\global\long\def\Ess#1{\ess\left[#1\right]}%

\global\long\def\fisher{\mathcal{F}}%

\global\long\def\bfield{\mathcal{B}}%
 
\global\long\def\borel{\mathcal{B}}%

\global\long\def\bernpdf{\text{Bernoulli}}%

\global\long\def\betapdf{\text{Beta}}%

\global\long\def\dirpdf{\text{Dir}}%

\global\long\def\gammapdf{\text{Gamma}}%

\global\long\def\gaussden#1#2{\text{Normal}\left(#1, #2 \right) }%

\global\long\def\gauss{\mathbf{N}}%

\global\long\def\gausspdf#1#2#3{\text{Normal}\left( #1 \lcabra{#2, #3}\right) }%

\global\long\def\multpdf{\text{Mult}}%

\global\long\def\poiss{\text{Pois}}%

\global\long\def\poissonpdf{\text{Poisson}}%

\global\long\def\pgpdf{\text{PG}}%

\global\long\def\wshpdf{\text{Wish}}%

\global\long\def\iwshpdf{\text{InvWish}}%

\global\long\def\nwpdf{\text{NW}}%

\global\long\def\niwpdf{\text{NIW}}%

\global\long\def\studentpdf{\text{Student}}%

\global\long\def\unipdf{\text{Uni}}%

\global\long\def\transp#1{\transpose{#1}}%
 
\global\long\def\transpose#1{#1^{\mathsf{T}}}%

\global\long\def\mgt{\succ}%

\global\long\def\mge{\succeq}%

\global\long\def\idenmat{\mathbf{I}}%

\global\long\def\trace{\mathrm{tr}}%

\global\long\def\argmax#1{\underset{_{#1}}{\text{argmax}} }%

\global\long\def\argmin#1{\underset{_{#1}}{\text{argmin}\ } }%

\global\long\def\diag{\text{diag}}%

\global\long\def\norm{}%

\global\long\def\spn{\text{span}}%

\global\long\def\vtspace{\mathcal{V}}%

\global\long\def\field{\mathcal{F}}%
 
\global\long\def\ffield{\mathcal{F}}%

\global\long\def\inner#1#2{\left\langle #1,#2\right\rangle }%
 
\global\long\def\iprod#1#2{\inner{#1}{#2}}%

\global\long\def\dprod#1#2{#1 \cdot#2}%

\global\long\def\norm#1{\left\Vert #1\right\Vert }%

\global\long\def\entro{\mathbb{H}}%

\global\long\def\entropy{\mathbb{H}}%

\global\long\def\Entro#1{\entro\left[#1\right]}%

\global\long\def\Entropy#1{\Entro{#1}}%

\global\long\def\mutinfo{\mathbb{I}}%

\global\long\def\relH{\mathit{D}}%

\global\long\def\reldiv#1#2{\relH\left(#1||#2\right)}%

\global\long\def\KL{KL}%

\global\long\def\KLdiv#1#2{\KL\left(#1\parallel#2\right)}%
 
\global\long\def\KLdivergence#1#2{\KL\left(#1\ \parallel\ #2\right)}%

\global\long\def\crossH{\mathcal{C}}%
 
\global\long\def\crossentropy{\mathcal{C}}%

\global\long\def\crossHxy#1#2{\crossentropy\left(#1\parallel#2\right)}%

\global\long\def\breg{\text{BD}}%

\global\long\def\lcabra#1{\left|#1\right.}%

\global\long\def\lbra#1{\lcabra{#1}}%

\global\long\def\rcabra#1{\left.#1\right|}%

\global\long\def\rbra#1{\rcabra{#1}}%

%% file: 01_abstract.tex
\begin{abstract}
Reinforcement learning (RL) offers the potential for training generally capable agents that can interact autonomously in the real world.
However, one key limitation is the brittleness of RL algorithms to core hyperparameters and network architecture choice.
Furthermore, non-stationarities such as evolving training data and increased agent complexity mean that different hyperparameters and architectures may be optimal at different points of training.
This motivates AutoRL, a class of methods seeking to automate these design choices.
One prominent class of AutoRL methods is Population-Based Training (PBT), which have led to impressive performance in several large scale settings.
In this paper, we introduce two new innovations in PBT-style methods.
First, we employ trust-region based Bayesian Optimization, enabling full coverage of the high-dimensional mixed hyperparameter search space.
Second, we show that using a \emph{generational} approach, we can also learn both architectures and hyperparameters jointly on-the-fly in a single training run.
Leveraging the new highly parallelizable Brax physics engine, we show that these innovations lead to large performance gains, significantly outperforming the tuned baseline while learning entire configurations on the fly. Code is available at \url{https://github.com/xingchenwan/bgpbt}.
\end{abstract}

%% file: 02_intro.tex
\section{Introduction}
\gls{RL} \citep{Sutton1998} has proven to be a successful paradigm for training agents across a variety of domains and tasks~\citep{mnih-atari-2013, alphago, qtopt, nguyen2021deep}, with some believing it could be enough for training generally capable agents~\citep{rewardenough}.
However, a crucial factor limiting the wider applicability of \gls{RL} to new problems is the notorious sensitivity of algorithms with respect to their hyperparameters~\citep{deeprlmatters, andrychowicz2021what, Engstrom2020Implementation}, which often require expensive tuning.
Indeed, it has been shown that when tuned effectively, good configurations often lead to dramatically improved performance in large scale settings~\citep{bo_alphago}.

To address these challenges, recent work in \emph{Automated Reinforcement Learning} (AutoRL) \citep{autorl_survey} has shown that rigorously searching these parameter spaces can lead to previously unseen levels of performance, even capable of breaking widely used simulators~\citep{pbtbt}.
However, {AutoRL} contains unique challenges, as different tasks even in the same suite are often best solved with different network architectures and hyperparameters~\citep{furuta2021pic, xu2022accelerated}.
Furthermore, due to the non-stationarities present in \gls{RL}~\citep{igl2021transient}, such as changing data distributions and the requirement for agents to model increasingly complex behaviors over time, optimal hyperparameters and architectures may not remain constant.
To address this, works have shown adapting hyperparameters through time~\citep{paul2019fast, pbtbt, parkerholder2021tuning, PBT} and defining fixed network architecture schedules~\citep{czarnecki2018mixmatch} can be beneficial for performance.
However, architectures and hyperparameters are inherently linked~\citep{networkwidthlearningrate}, and to date, no method combines the ability to jointly and continuously adapt both on the fly.

In this paper we focus on \gls{PBT} \citep{PBT} methods, where a population of agents is trained in parallel, copying across stronger weights and enabling adaption of hyperparameters in a single training run.
This allows \gls{PBT} methods to achieve impressive performance on many large-scale settings~\citep{capturetheflag, pbtfootball}.
However, \gls{PBT}-style methods are typically limited in scope due to two key factors: 1) they only optimize a handful of hyperparameters, either due to using random search \citep{PBT}, or model-based methods that do not scale to higher dimensions~\citep{pb2, parkerholder2021tuning};
2) \gls{PBT} methods are usually restricted to the same fixed architecture since weights are copied between agents.

\begin{figure}[t]
    \begin{center}
        \includegraphics[trim=0cm 0cm 0cm  -1.1cm, clip, width=1.0\linewidth]{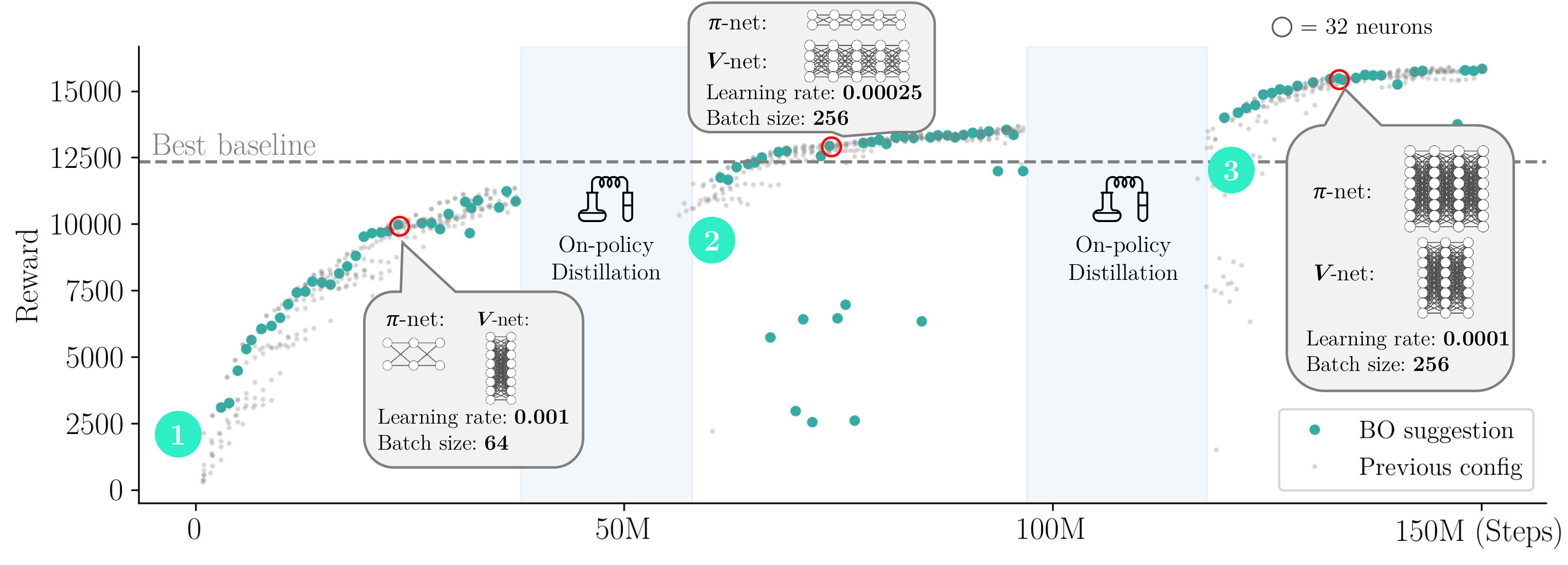}   
    \end{center}
    \vspace{-5mm}
        \caption{An example run of \textsc{bg-pbt} on HalfCheetah task in \textsc{Brax}: \textsc{bg-pbt} combines population-based training with high-dimensional Bayesian optimization, generational training (different generations marked with numbers in the figure) and on-policy distillation between generations to transfer across \gls{RL} agents with different neural architectures: at different points during training, both hyperparameters and the architectures of policy \& value networks are tuned on-the-fly, leading to significant improvement over the baseline.
        }
        \label{fig:cover_fig}
        \vspace{-8mm}
\end{figure}

We seek to overcome both of these issues in this paper, and propose \emph{\ourNAS}, with an example run demonstrated in Fig.~\ref{fig:cover_fig}.
\ourNAS is capable of tuning a significantly greater proportion of the agent's configuration, thanks to two new ideas.
First, we introduce a new model-based \emph{hyperparameter and architecture} exploration step motivated by recent advances in local Bayesian optimization~\citep{casmopolitan}.
Second, we take inspiration from \citet{openendedlearningteam2021openended} who showed that \gls{PBT} can be particularly effective when combined with network distillation~\citep{igl2021transient}, in an approach known as generational learning.
As prior works in generational training~\citep{alphastar, openendedlearningteam2021openended} show, the use of successive generations of architectures with distillation results in significantly reduced training time for new agents.
This provides us with an algorithm-agnostic framework to create agents which continuously discover their \emph{entire configuration}.
Thus, for the first time, we can tune hyperparameters and architectures during one training run as part of a single unified algorithm.

We run a series of exhaustive experiments tuning both the architectures and hyperparameters for a \gls{PPO}~\citep{ppo}  agent in the newly introduced \textsc{Brax} environment suite~\citep{brax}.
\textsc{Brax} enables massively parallel simulation of agents, making it perfect for testing population-based methods without vast computational resources.
Our agents significantly outperform both the tuned baseline and a series of prior \gls{PBT} methods.
Notably, we observe that \ourNAS often discovers a \emph{schedule of networks} during training---which would be infeasible to train from scratch.
Furthermore, \ourNAS discovers entirely new modes of behavior for these representative environments, which we show at \url{https://sites.google.com/view/bgpbt}. 

To summarize, the main contributions of this paper are as follows:
\begin{enumerate}[noitemsep,topsep=-8pt, leftmargin=3pt]
    \item We show for the first time it is possible to select architectures as part of a \textcolor{blue}{general-purpose} \gls{PBT} framework, using \textit{generational training with policy distillation} with \gls{NAS}.
    \item We propose a novel and efficient algorithm, \textbf{\ourNAS}, especially designed for high-dimensional mixed search spaces, which can select both architectures and hyperparameters on-the-fly with provable efficiency guarantees.
    \item We show in a series of experiments our \textit{automatic architecture curricula} make it possible to achieve significantly higher performance than previous methods.
\end{enumerate}

%% file: 03_background.tex
\section{Preliminaries}
We begin by introducing the reinforcement learning framework, population-based training, which our method is based on, and the general problem setup we investigate in this paper.

\textbf{Reinforcement Learning.}
We model the environment as a \gls{MDP} \citep{Sutton1998}, defined as a tuple $M = (\mathcal{S}, \mathcal{A}, P, R, \rho_0, \gamma)$, where $\mathcal{S}$ and $\mathcal{A}$ denote the state and action spaces respectively, $P(s_{t+1} | s_t, a_t)$  the transition dynamics, $R(s_t, a_t)$  the reward function, $ \rho_0$  the initial state distribution, and $\gamma \in (0, 1)$ the discount factor.
The goal is to optimize a policy $\pi (a_t | s_t)$ that maximizes the expected discounted return $\mathbb{E}_{\pi, P, \rho_0}\left[\sum_{t=0}^\infty \gamma^t R(s_t, a_t)\right]$.
Given a policy $\pi$, we may define the state value function $V^{\pi}(s) = \mathbb{E}_{\pi, P}\left[\sum_{t=0}^\infty \gamma^t R(s_t, a_t) | s_0 = s \right]$ and the state-action value-function $Q^{\pi}(s, a) = \mathbb{E}_{\pi, P}\left[\sum_{t=0}^\infty \gamma^t R(s_t, a_t) | s_0 = s, a_0 = a \right]$. The advantage function is then defined as the difference $A^{\pi}(s, a) = Q^{\pi}(s, a) - V^{\pi}(s)$.

A popular algorithm for online continuous control that we use is \gls{PPO} \citep{ppo}.
\gls{PPO} achieves state-of-the-art results for popular benchmarks~\citep{procgen} and is hugely parallelizable, making it an ideal candidate for population-based methods~\citep{pb2}.
\gls{PPO} approximates \textsc{trpo} ~\citep{ppo, trpo} and uses a clipped objective to stabilize training:
\begin{align}
\label{eqn:ppo}
\mathcal{L}_{\mathrm{PPO}}(\theta)=\min \left(\frac{\pi_{\theta}(a \mid s)}{\pi_{\mu}(a \mid s)} A^{\pi_{\mu}}, g(\theta, \mu) A^{\pi_{\mu}}\right), \text { where } g(\theta, \mu)=\operatorname{clip}\left(\frac{\pi_{\theta}(a \mid s)}{\pi_{\mu}(a \mid s)}, 1-\epsilon, 1+\epsilon\right)
\end{align}
where $\pi_{\mu}$ is a previous policy and $\epsilon$ is the clipping parameter.

\textbf{Population-Based Training.}
\gls{RL} algorithms, including \gls{PPO}, are typically quite sensitive to their hyperparameters.
\gls{PBT} \citep{PBT} is an evolutionary method that tunes \gls{RL} hyperparameters on-the-fly.
It optimizes a population of $B$ agents in parallel, so that their weights and hyperparameters may be dynamically adapted within a single training run.
In the standard paradigm without architecture search, we consider two sub-routines, $\mathrm{explore}$ and $\mathrm{exploit}$.
We train for a total of $T$ steps and evaluate performance every $t_{\mathrm{ready}} < T$ steps.
In the $\mathrm{exploit}$ step, the weights of the worst-performing agents are replaced by those from an agent randomly sampled from the set of best-performing ones, via \emph{truncation selection}.
To select new hyperparameters, we perform the $\mathrm{explore}$ step.
We denote the hyperparameters for the $b$th agent in a population at timestep $t$ as $\bz_t^b\in\mathcal{Z}$;
this defines a \textit{schedule} of hyperparameters over time $\left( z_t^b \right)_{t=1,...T}$.
Let $f_t(z_t)$ be an objective function (e.g. the return of a \gls{RL} agent) under a given set of hyperparameters at timestep $t$, our goal is to maximize the final performance $f_T(z_T)$.

The original \gls{PBT} uses a combination of random sampling and evolutionary search for the explore step by suggesting new hyperparameters mutated from the best-performing agents.
\gls{PB2} and \gls{PB2}-Mix \citep{pb2, parkerholder2021tuning} improve on \gls{PBT} by using \emph{\gls{BO}} to suggest new hyperparameters, relying on a time-varying \emph{\gls{GP}} \citep{RasmussenGP, bogunovic2016time} to model the data observed.
We will also use \gls{GP}-based \gls{BO} in our method, and we include a primer of \gls{GP}s and \gls{BO} in App. \ref{app:primer}.

\textbf{Problem Setup.}
We follow the notation used in \cite{pb2} and frame the hyperparameter optimization problem in the lens of optimizing an expensive, time-varying, black-box reward function $f_t : \mathcal{Z} \rightarrow \mathbb{R}$.
Every $t_{\mathrm{ready}}$ steps, we observe and record noisy observations, $y_t = f_t(\bz_t) + \epsilon_t$, where $\epsilon_t \sim \mathcal{N}(0, \sigma^2\mathbf{I})$ for some fixed variance $\sigma^2$.
We follow the typical \gls{PBT} setup by defining a hyperparameter space, $\mathcal{Z}$, which for the \textsc{Brax}~\citep{brax} implementation of \gls{PPO} we follow in the paper, 
consists of 9 parameters: learning rate, discount factor ($\gamma$), entropy coefficient ($c$), unroll length, reward scaling, batch size, updates per epoch, GAE parameter ($\lambda$) and clipping parameter ($\epsilon$).
To incorporate the architecture hyperparameters, $\by \in \mathcal{Y}$, we add 6 additional parameters leading to a 15-dimensional joint space $\mathcal{J} = \mathcal{Y} \times \mathcal{Z}$.
For both the policy and value networks, we add the width and depth of the \gls{MLP} and a binary flag on whether to use spectral normalization.

%% file: 05_method.tex
\section{Bayesian Generational Population-Based Training (BG-PBT)}
\begin{wrapfigure}{R}{0.5\textwidth}
\vspace{-7.5mm}
    \begin{minipage}{0.5\textwidth}
\begin{algorithm}[H]
\begin{footnotesize}
	    \caption{\ourNAS;
	    distillation and \gls{NAS} steps marked in \textcolor{magenta}{magenta} (§\ref{subsec:distillation})
	    }
	    \label{alg:main_alg}
	\begin{algorithmic}[1]
		\STATE {\bfseries Input}: pop size $B$, $t_{\mathrm{ready}}$, max steps $T$, $q$ (\% agents replaced per iteration)
		\STATE \textbf{Initialize} $B$ agents with weights $\{\theta_0^{(i)}\}_{i=1}^B$, random hyperparameters $\{\bz_0^{(i)}\}_{i=1}^B$ \textcolor{magenta}{and architectures $\{\by_0^{(i)}\}_{i=1}^B$},  %
        \FOR{$t=1, \dots, T$ (in parallel for all $B$ agents)}
        \STATE Train models \& record data for all agents 
        \IF {$t~\mathrm{mod}~t_{\mathrm{ready}} = 0$}
            \STATE Replace the weights \textcolor{magenta}{\& architectures} of the bottom $q \%$ agents with those of the top $q \%$ agents.
            \STATE Update the surrogate with new observations \& returns and adjust/restart the trust regions.
            \STATE \textcolor{magenta}{Check whether to start a new generation (see §\ref{subsec:distillation}).}%
            \IF {\textcolor{magenta}{start a new generation}}
            \STATE \textcolor{magenta}{Clear the \gls{GP} training data.}
            \STATE \textcolor{magenta}{Create $B$ agents with archs.~from BO/random.}
            \STATE \textcolor{magenta}{Distill from a top-$q\%$ performing agent of the existing generation to new agents.}
            \ELSE \STATE{Select new hyperparameters $\bz$ for the agents whose weights have been just replaced with randomly sampled configs (if $\mathbf{D} = \emptyset$) \textbf{OR} using the suggestions from the \gls{BO} agent described \textcolor{magenta}{conditioned on $\by$} (otherwise)}.
            \ENDIF
        \ENDIF
    	\ENDFOR
	\end{algorithmic}
\end{footnotesize}
\end{algorithm}
\end{minipage}
\vspace{-9mm}
\end{wrapfigure}
We present \our in \cref{alg:main_alg} which consists of two major components. First, a \gls{BO} approach to select new hyperparameter configurations $\bz$ for our agents (§\ref{subsec:bo}).
We then extend the search space to accommodate architecture search, allowing agents to choose their own networks (parameterized by $\by \in \mathcal{Y}$) and use on-policy distillation to transfer between different architectures (§\ref{subsec:distillation}). 

\subsection{High-Dimensional BO Agents in Mixed-Input Configuration Space for PBT}
\label{subsec:bo}

Existing population-based methods ignore (\gls{PB2}) or only partially address (\gls{PB2}-Mix, which does not consider ordinal variables such as integers) the heterogeneous nature of the mixed hyperparameter space $\mathcal{Z}$.
Furthermore, both previous methods are equipped with standard \gls{GP} surrogates which typically scale poorly beyond low-dimensional search spaces, and are thus only used to tune a few selected hyperparameters deemed to be the most important based on human expertise.
To address these issues, \our explicitly accounts for the characteristics of typical \gls{RL} hyperparameter search space by making several novel extensions to \textsc{Casmopolitan} \citep{casmopolitan}, a state-of-the-art \gls{BO} method for high-dimensional, mixed-input problems for our setting.
In this section, we outline the main elements of our design, and we refer the reader to  App. \ref{appsubsec:propose} for full technical details of the approach.

\textbf{Tailored Treatment of Mixed Hyperparameter Types.}
Hyperparameters in \gls{RL} can be continuous (e.g. discounting factor), ordinal (discrete variables with ordering, e.g. batch size) and categorical (discrete variables without ordering, e.g. activation function).
\our treats each variable type differently: we use tailored kernels for the \gls{GP} surrogate, and utilize interleaved optimization for the acquisition function, alternating between local search for the categorical/ordinal variables and gradient descent for the continuous variables.
\our extends both \textsc{Casmopolitan} and \gls{PB2}-Mix by further accommodating ordinal variables, as both previous works only considered continuous and categorical variables.
We demonstrate the considerable benefits of explicitly accounting for the ordinal variables in App. \ref{app:ablation_ordinal}.

\textbf{\gls{TR}.} 
\gls{TR}s have proven success in extending \gls{GP}-based \gls{BO} to higher-dimensional search spaces, which were previously intractable due to the curse of dimensionality, by limiting exploration to promising regions in the search space based on past observations \citep{eriksson2019scalable, casmopolitan}.
In the \gls{PBT} context, \gls{TR}s also implicitly avoid large jumps in hyperparameters, which improves training stability. 
We adapt the \gls{TR}s used in the original \textsc{Casmopolitan} to the time-varying setup by defining \gls{TR}s around the current best configuration, and then adjusting them dynamically: 
similar to \citet{eriksson2019scalable} and \citet{casmopolitan}, \gls{TR}s are expanded or shrunk upon consecutive ``successes'' or ``failures''.
We define a proposed configuration to be a ``success'' if it appears in the top $q\%$-performing agents and a ``failure'' otherwise. 
When the \gls{TR}s shrink below some specified minimum, a \emph{restart} is triggered, which resets the \gls{GP} surrogate to avoid becoming stuck at a local optimum.
We adapt the \gls{UCB}-based criterion proposed in~\citet{casmopolitan} which is based on a global, auxiliary \gls{GP} model to the time-varying setting to re-initialize the population when a restart is triggered.
Full details are provided in App. \ref{appsubsec:restart_global}.

\textbf{Theoretical Properties.}
Following \citet{casmopolitan}, we show that under typical assumptions (presented in App. \ref{app:proofs}) used for \gls{TR}-based algorithms~\citep{yuan2000review}, our proposed \ourNAS, without architecture and distillation to be introduced in §\ref{subsec:distillation}, converges to the global optimum asymptotically.
Furthermore, we derive an upper bound on the cumulative regret and show that under certain conditions it achieves sublinear regret.
We split the search space into $\mathcal{Z} = [\mathcal{H}, \mathcal{X}]$ (categorical/continuous parts respectively).
We note that \cref{assu:gp-approx} considers the minimum \gls{TR} lengths $L^x_{\min}, L^h_{\min}$ are set to be small enough so that the \gls{GP} approximates $f$ accurately in the \gls{TR}s. 
In practice, this assumption only holds asymptotically, i.e. when the observed datapoints in the \gls{TR}s goes to infinity.
We present the main result, the time-varying extension to Theorem 3.4 from \citet{casmopolitan}, and then refer to App. \ref{app:proofs} for the derivation.

\begin{theorem}
\label{thr:gconverge-mix}
Assume Assumptions \ref{assu:f-bounded} \& \ref{assu:gp-approx} hold.
Let $f_i: [\mathcal{H}, \mathcal{X}] \rightarrow \mathbb{R}$ be a time-varying objective defined over a mixed space and $\zeta \in (0, 1)$.
Suppose that:
    (i) there exists a class of functions $g_i$ in the RKHS $\mathcal{G}_k([\mathcal{H}, \mathcal{X}])$ corresponding to the kernel $k$ of the global \gls{GP} model, such that $g_i$ passes through all the local maximas of $f_i$ and shares the same global maximum as $f_i$;
    (ii) the noise at each timestep $\epsilon_i$ has mean zero conditioned on the history and is bounded by $\sigma$;
    (iii) $\Vert g_i \Vert^2_k \leq B$.
Then \ourNAS obtains a regret bound
\begin{equation} \nonumber
\text{Pr} \Big\lbrace R_{IB} \leq  \sqrt{\frac{C_1 I\beta_{I}}{B}
\gamma \bigl(IB;k;[\mathcal{H}, \mathcal{X}] \bigr)}+2 \quad \forall I \geq 1 \Big \rbrace \geq 1 - \zeta,
\end{equation}
with $C_1=8/\log(1+\sigma^{-2})$, $\gamma(T;k;[\mathcal{H}, \mathcal{X}])$ defined in \cref{thr:tv_kernel_mig} and $\beta_I$ is parameter balancing  exploration-exploitation as in Theorem 2 of \citet{pb2}.
\end{theorem}

Under the same ideal conditions assumed in~\citet{bogunovic2016time, pb2} where the objective does not vary significantly through time, the cumulative regret bound is sublinear with $\lim_{I \rightarrow \infty} \frac{R_{IB}}{I}=0$, when $\omega \rightarrow 0$ and $\tilde{N} \rightarrow I$.

\begin{figure}[t]
\vspace{-4mm}
    \begin{center}
        \includegraphics[trim=0cm 0cm 0cm  -1.1cm, clip, width=1.0\linewidth]{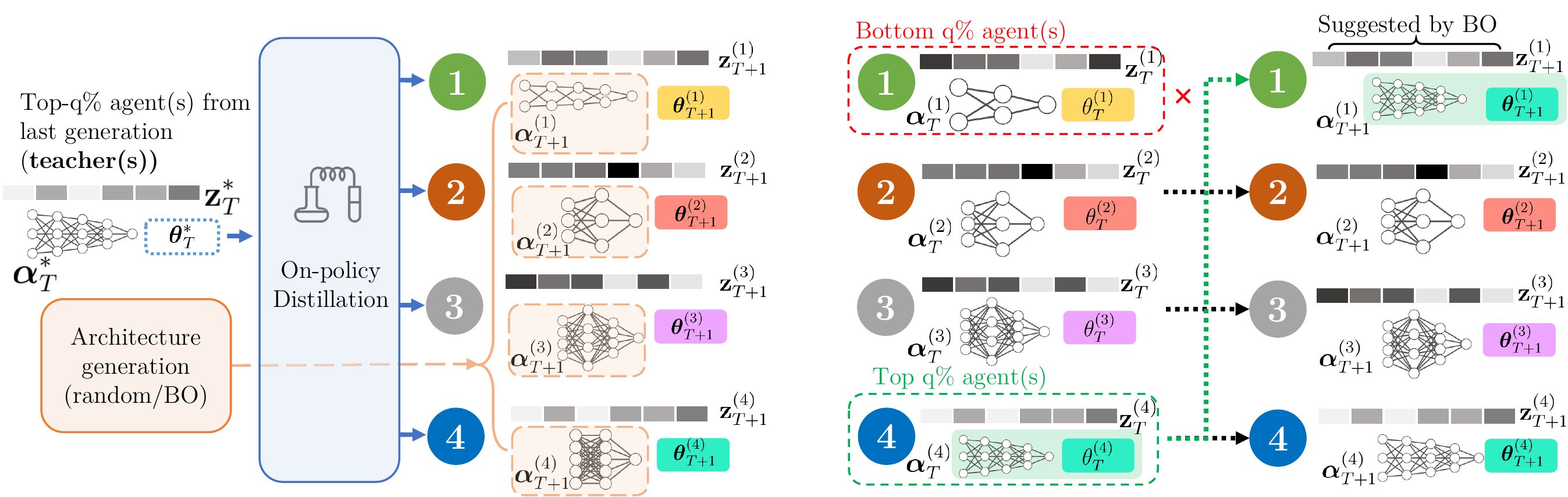}   
    \end{center}
    \vspace{-4mm}
        \caption{\our (a) at the \emph{beginning} of a generation (\textbf{left}) and (b) \emph{during} a generation (\textbf{right}). At the start of a generation, agents with diverse architectures are suggested and on-policy distillation is used to transfer information across generations \& different architectures (§\ref{subsec:distillation}). Within a generation, a high-dimensional, mixed-input \gls{BO} agent suggests hyperparameters (§\ref{subsec:bo}, we copy weights across fixed architectures).
        }
        \label{fig:method_picture}
        \vspace{-4mm}
\end{figure}

\subsection{Adapting Architectures on the Fly}
\label{subsec:distillation}
Now that we are equipped with an approach to optimize in high-dimensional $\mathcal{Z}$, we focus on choosing \emph{network architectures}.
Despite their importance in \gls{RL} \citep{coinrun, furuta2021pic}, architectures remain underexplored as a research direction.
Adapting architectures for \gls{PBT} methods is non-trivial as we further enlarge the search space, and weights cannot readily be copied across different networks.
Inspired by \citet{openendedlearningteam2021openended}, our key idea is that when beginning a new {\emph{generation}} we can distill behaviors into new architectures (see Fig. \ref{fig:method_picture}).
Specifically:
\begin{itemize}[noitemsep,topsep=-8pt, leftmargin=3pt]
\item
\emph{Starting each generation:} We fill the population of $B$ agents by generating a diverse set of architectures for both the policy and value networks.
For the first generation, this is done via random sampling.
For subsequent generations, we use suggestions from \gls{BO} and/or random search with successive halving over the architecture space $\mathcal{Y}$ only (refer to App. \ref{appsubsec:new_archs} for details); the \gls{BO} is trained on observations of the best performance each architecture has achieved in previous generations.
We initialize a new generation when the evaluated return stagnates beyond a pre-set patience during training.

\item
\emph{Transfer between generations:} 
Apart from the very first generation, we transfer information from the best agent(s) of the previous generation to each new agent, in a similar fashion  to~\citet{openendedlearningteam2021openended}, using on-policy distillation with a joint supervised and \gls{RL} loss between \emph{different architectures} as shown in Fig.~\ref{fig:method_picture}a.
Given a learned policy $\pi_i$ and value function $V_i$ from a previous generation, the new joint loss optimized is:
\begin{align}
\mathbb{E}_{(s_t, a_t) \sim \pi_{i+1}} [\alpha_{\mathrm{RL}}\mathcal{L}_{\mathrm{RL}} + \alpha_{V}\norm{V_i(s_t) - V_{i+1}(s_t)}_2 + \alpha_{\pi} \mathbb{D}_{\textrm{KL}}\bigl( \pi_i(\cdot \vbar s_t) \,||\, \pi_{i+1}(\cdot \vbar s_t) \bigr)]
\end{align}
for weights $\alpha_{\mathrm{RL}}\geq0$, $\alpha_{V}\geq0$, $\alpha_{\pi}\geq0$, and \gls{RL} loss $\mathcal{L}_{\mathrm{RL}}$ taken from \cref{eqn:ppo}.
We linearly anneal the supervised losses over the course of each generation, so that by the end, only the \gls{RL} loss remains.

\item
\emph{During a generation:} We follow standard \gls{PBT} methods to evolve the hyperparameters of each agent by copying weights $\btheta$ \emph{and the architecture} $\by$ from a top-$q\%$ performing agent to a bottom-$q\%$ agent, as shown in Fig.~\ref{fig:method_picture}b.
This creates an effect similar to successive halving \citep{pmlr-v28-karnin13, pmlr-v51-jamieson16} where poorly-performing architectures are quickly removed from the population in favor of more strongly-performing ones; typically at the end of a generation, 1 or 2 architectures dominate the population.
While we do not introduce new architectures within a generation, the hyperparameter suggestions are conditioned on the current policy and value architectures by incorporating the architecture parameters $\by$ as contextual fixed dimensions in the \gls{GP} surrogate described in §\ref{subsec:bo}.
\end{itemize}

%% file: 06_experiments.tex
\section{Experiments}
\label{sec:experiments}
While \our provides a framework applicable to any \gls{RL} algorithm, we test our method on 7 environments from the new \textsc{Brax} environment suite, using \gls{PPO}.
We begin by presenting a comparative evaluation of \our against standard baselines in population-based training to both show the benefit of searching over the full hyperparameter space with local \gls{BO} and of automatically adapting architectures over time.
We further show that our method beats end-to-end \gls{BO}, showing the advantage of dynamic schedules.
Next, we analyze these learned hyperparameter and architecture schedules using \our and we show analogies to similar trends in learning rate and batch size in supervised learning.
Finally, we perform ablations on individual components of \our.
For all population-based methods, we use a population size $B=8$ and a total budget of 150M steps.
We note that \our with architectures uses additional on-policy samples from the environment in order to distill between architectures.
We instantiate the \textsc{Brax} environments with an action repeat of 1.
We use $t_{\mathrm{ready}}$ of 1M for all \gls{PBT}-based methods on all environments except for Humanoid and Hopper, where we linearly anneal $t_{\mathrm{ready}}$ from 5M to 1M.
The remaining hyperparameters and implementation details used in this section are listed in App. \ref{app:implementation_details}.

\begin{figure}[t]
\vspace{-2mm}
\centering
     \begin{subfigure}{1\linewidth}
          \includegraphics[width=\textwidth]{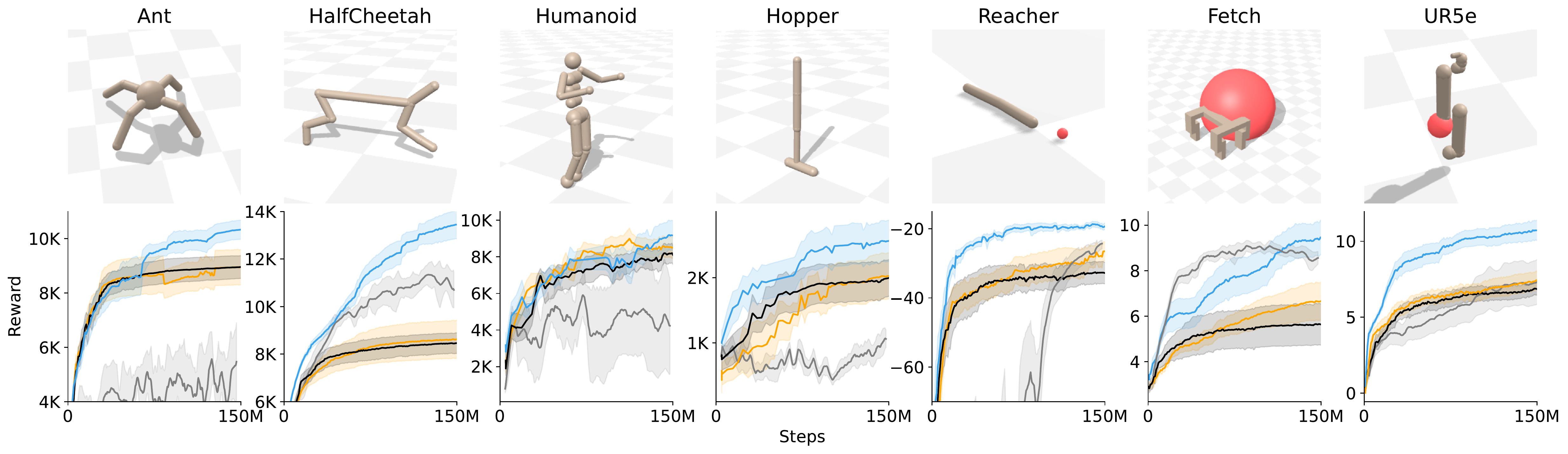}
     \end{subfigure}
     \vspace{-2mm}
     \begin{subfigure}{0.4\linewidth}
       \includegraphics[width=\textwidth, trim={0, 0, 0, .6cm}, clip]{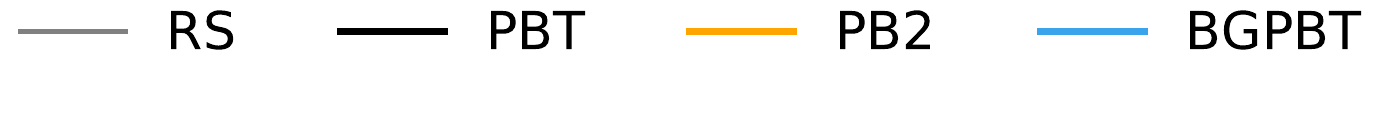}
     \end{subfigure}
    \vspace{-3.5mm}
\caption{
    \small{Visualization of each environment (\textbf{top row}) and mean evaluated return over the population with $\pm$1 \textsc{sem} (shaded) across 7 random seeds (\textbf{bottom row}) in all environments. \texttt{RS} refers to the higher performing of RS-$\mathcal{Z}$ or RS-$\mathcal{J}$ in \cref{tab:main_results}.
    }
}
\vspace{-4mm}
\label{fig:main_fig}
\end{figure}
\begin{table}[tb]
    \centering
        \caption{{Mean evaluated return $\pm 1$\textsc{sem} across 7 seeds shown. For \gls{PBT}-style methods (\gls{PBT}, \gls{PB2} and \our), the mean best-performing agent in the population is shown. Methods performing within 1 \textsc{sem} of the best-performing method are bolded (the same applies to all tables).
        }
        }
        \vspace{-3mm}
        \begin{footnotesize}
    \adjustbox{max width=0.98\columnwidth}{
    \begin{tabular}{ccccccc}
    \toprule
     Method & \textit{PPO}$^*$ & RS & RS & PBT & PB2 & BG-PBT\\
     Search space & $\mathcal{Z}$ & $\mathcal{Z}$ & $\mathcal{J}$ & $\mathcal{Z}$ & $\mathcal{Z}$ & $\mathcal{J}$\\
    \midrule
    Ant & ${3853}_{\pm 676}$ & $6780_{\pm 317}$ & $4781_{\pm 515}$ & $8955_{\pm 385}$ &  $8954_{\pm 594}$ & $\mathbf{10349_{\pm 326}}$ \\
    HalfCheetah & ${6037}_{\pm 236}$ & $9502_{\pm 76}$ & $10340_{\pm 329}$ & $8455_{\pm 400}$ & $8629_{\pm 746}$ & $\mathbf{13450_{\pm 551}}$ \\
    Humanoid  & ${\boldsymbol{{9109}}_{\pm 987}}$ & $4004_{\pm 519}$ & $4652_{\pm 1002}$ & $7954_{\pm 437}$ & ${8452_{\pm 512}}$ & $\mathbf{9171_{\pm 748}}$ \\
    Hopper  & ${120}_{\pm 43}$ & $339_{\pm 25}$ & $943_{\pm 185}$ & $2002_{\pm 254}$ & $2027_{\pm 323}$ &  $\mathbf{2569_{\pm 293}}$\\
    Reacher & ${-189.3}_{\pm 43.7}$ & $-24.2_{\pm 1.4}$ & $-95.2_{\pm 25.3}$ & $-32.9_{\pm 2.8}$ & $-26.6_{\pm 2.6}$ & $\mathbf{{-19.2_{\pm 0.9}}}$\\
    Fetch & $\boldsymbol{{14.0}_{\pm 0.2}}$ & $5.2_{\pm 0.4}$ & $8.6_{\pm 0.2}$ & $5.5_{\pm 0.8}$ & $6.6_{\pm 0.7}$ &  ${9.4}_{\pm 0.7}$ \\    
    UR5e & ${5.2}_{\pm 0.2}$ & $5.3_{\pm 0.4}$ & $7.7_{\pm 0.3}$ & $6.9_{\pm 0.4}$ & $7.4_{\pm 0.6}$ & $\mathbf{10.7_{\pm 0.6}}$\\
    \bottomrule 
    \multicolumn{7}{l}{$^*$From the \textsc{Brax} authors and implemented in a different framework (JAX) to ours (PyTorch)}
    \end{tabular}}
    \label{tab:main_results}
    \end{footnotesize}
    \vspace{-5mm}
\end{table}

\textbf{Comparative Evaluation of \our.}
We first perform a comparative evaluation of \our against standard baselines in \gls{PBT}-methods and the \gls{PPO} baseline provided by the \textsc{Brax} authors.
We show the benefit of using local \gls{BO} and treating the whole \gls{RL} hyperparameter space $\mathcal{Z}$, by comparing \our against \gls{PBT}~\citep{PBT}, \gls{PB2}~\citep{pb2} and \gls{RS} using the default architecture in \textsc{Brax}.
In \gls{RS}, we simply sample from the hyperparameter space and take the best performance found using the same compute budget as the \gls{PBT} methods.
Next, we include architecture search into \our using the full space $\mathcal{J}$ and show significant gains in performance compared to \our without architectures; we use random search over $\mathcal{J}$ as a baseline.
The optimized \gls{PPO} implementation from the \textsc{Brax} authors is provided as a sequential baseline.
We present the results in \cref{tab:main_results} and the training trajectories in Fig.~\ref{fig:main_fig}.

We show that \our significantly outperforms the \gls{RS} baselines and the existing \gls{PBT}-style methods in almost all environments considered.
We also observe that \gls{RS} is a surprisingly strong baseline, performing on par or better than \gls{PBT} and \gls{PB2} in HalfCheetah, Reacher, Fetch and UR5e --- this is due to a well-known failure mode in \gls{PBT}-style algorithms where they may be overly greedy in copying sub-optimal behaviors early on and then fail to sufficiently explore in weight space when the population size is modest.
\our avoids this problem by \emph{re-initializing networks each generation} and distilling, which prevents collapse to suboptimal points in weight space.

\textbf{Experiments with Different Training Timescales.}
We further show experiments with a higher budget of $300$M timesteps and/or an increased population size up to $B=24$ to investigate the scalability of \our in larger-scale environments (App. \ref{app:increased_timesteps}). 
Furthermore, given that \our uses additional on-policy samples in order to distill between architectures, we conduct experiments of \our with reduced training budget in App. \ref{app:decreased_timesteps}.
We find that even when the maximum timesteps are roughly halved, \our still outperforms the baseline AutoRL methods. 

\textbf{Comparison Against Sequential \gls{BO}.} We further compare against \gls{BO} in the traditional sequential setup (Table \ref{tab:bo_comparison}): for each \gls{BO} iteration, the agent is trained for the full 150M timesteps before a new hyperparameter suggestion is made.
To enable \gls{BO} to improve on \gls{RS}, we allocate a budget of 50 evaluations, \emph{which is up to 6$\times$ more expensive} than our method and even more costly in terms of wall-clock time if vanilla, non-parallel \gls{BO} is used.
We implement this baseline using SMAC3 \citep{lindauer2022smac3} in both the $\mathcal{Z}$ and $\mathcal{J}$ search spaces (denoted BO-$\mathcal{Z}$ and BO-$\mathcal{J}$ respectively in \cref{tab:bo_comparison}).
While, unsurprisingly, \gls{BO} improves over the \gls{RS} baseline, \our still outperforms it in a majority of environments.
One reason for this is that \our naturally discovers a dynamic schedule of hyperparameters and architectures, which is strictly more flexible than a carefully tuned but still static configuration -- we analyze this below.

\begin{table}[t]
\vspace{-6mm}
\centering
\begin{minipage}{.47\linewidth}
\centering
  \captionsetup{width=.91\linewidth}
        \caption{\small{Comparison against sequential BO$^*$}
        }
        \vspace{-3mm}
        \begin{footnotesize}
    \adjustbox{max width=0.95\columnwidth}{
    \begin{tabular}{cccc}
    \toprule
     Method & BO-$\mathcal{Z}$$^*$ & BO-$\mathcal{J}$$^*$ & BG-PBT\\
    \midrule
    Ant &  $6975_{\pm 1013}$ & $7149_{\pm 507}$ & $\mathbf{10349_{\pm 326}}$ \\
    HalfCheetah & $11202_{\pm 204}$ & $10859_{\pm 174}$ &$\mathbf{13450_{\pm 551}}$ \\
    Humanoid  &  $\mathbf{9040_{\pm 1303}}$ & ${4845_{\pm 962}}$ & $\mathbf{{9171_{\pm 748}}}$ \\
    Hopper  &  $358_{\pm 60}$ & ${1254_{\pm 154}}$ &  $\mathbf{2569_{\pm 293}}$\\
    Reacher & $\mathbf{-17.3_{\pm 0.3}}$ & ${{-51.7_{\pm 18.3}}}$ &  ${{-19.2_{\pm 0.9}}}$\\
    Fetch & $\mathbf{13.2_{\pm 0.2}}$ & $11.6_{\pm 0.1}$ & ${9.4}_{\pm 0.7}$ \\    
    UR5e & $9.0_{\pm 0.5}$ & $6.3_{\pm 1.4}$ & $\mathbf{10.7_{\pm 0.6}}$\\
    \bottomrule 
    \multicolumn{4}{l}{$^*$More resources required compared to \our.}
    \end{tabular}}
    \label{tab:bo_comparison}
    \end{footnotesize}
\end{minipage}
\hfill
\begin{minipage}{.49\linewidth}
\centering
  \captionsetup{width=.95\linewidth}
        \caption{Ablation studies
        }
        \vspace{-3mm}
        \begin{footnotesize}
    \adjustbox{max width=0.95\columnwidth}{
    \begin{tabular}{cccc}
    \toprule
     Method & No TR/NAS & No NAS & BG-PBT\\
    \midrule
    Ant &  $8954_{\pm 594}$ & $9352_{\pm 402}$ & $\mathbf{10349_{\pm 326}}$ \\
    HalfCheetah & $8629_{\pm 746}$ & $9483_{\pm 626}$ &$\mathbf{13450_{\pm 551}}$ \\
    Humanoid  &  ${8452_{\pm 512}}$ & $\mathbf{10359_{\pm 647}}$ & ${9171_{\pm 748}}$ \\
    Hopper  &  $2027_{\pm 323}$ & $\mathbf{{2511_{\pm 154}}}$ &  $\mathbf{2569_{\pm 293}}$\\
    Reacher & $-26.6_{\pm 2.6}$ & $\mathbf{{-17.6_{\pm 0.8}}}$ &  ${{-19.2_{\pm 0.9}}}$\\
    Fetch & $6.6_{\pm 0.7}$ & $7.3_{\pm 0.8}$ & $\mathbf{9.4}_{\pm 0.7}$ \\    
    UR5e & $7.4_{\pm 0.6}$ & $9.0_{\pm 0.8}$ & $\mathbf{10.7_{\pm 0.6}}$\\
    \bottomrule 
    \\
    \end{tabular}}
    \label{tab:ablation}
    \end{footnotesize}
\end{minipage}
\end{table}
\begin{figure}[tb]
\vspace{-4mm}
\centering
     \centering
     \includegraphics[width=\textwidth]{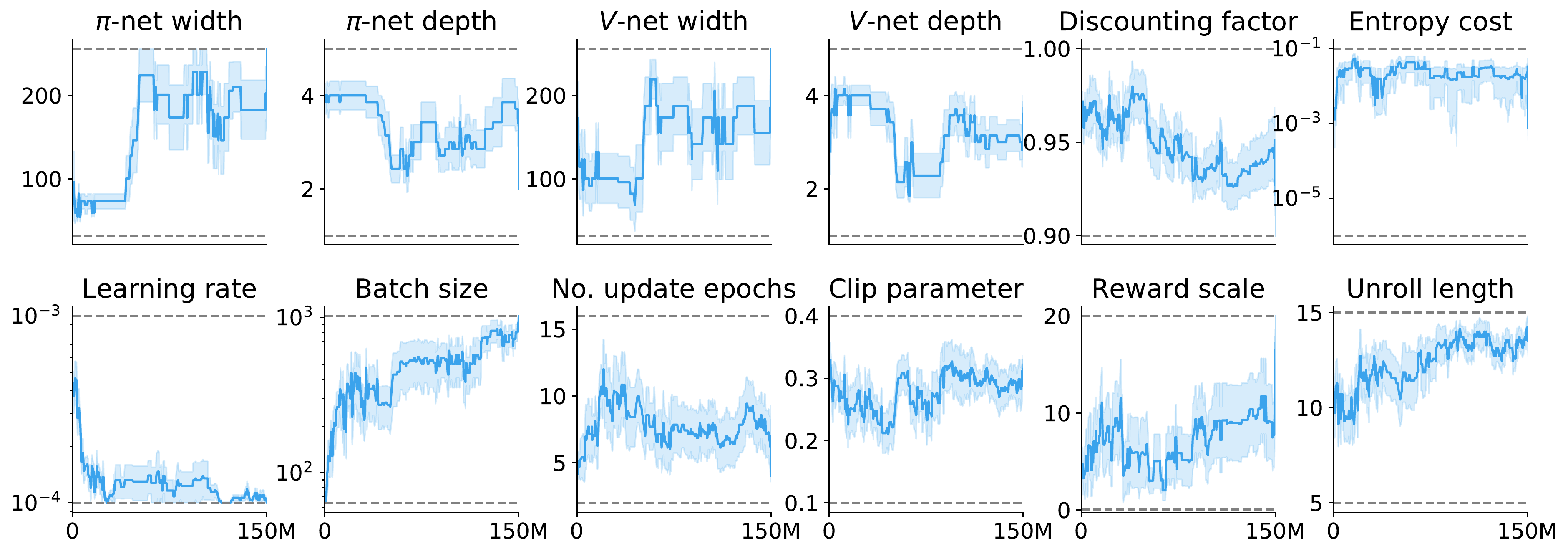}
\vspace{-7mm}
\caption{
    \small{The hyperparameter and architecture schedule discovered by \our on Ant: we plot the hyperparameters of the best-performing agent in the population averaged across 7 seeds with $\pm$ 1 \textsc{sem} shaded. \textcolor{gray}{Gray dashed lines} denote the hyperparameter bounds.
    }
}
\vspace{-5mm}
\label{fig:ant_schedule}
\end{figure}

\textbf{Analysis of Discovered Hyperparameter and Architecture Schedules.}
We present the hyperparameter and architecture schedules learned by \our in our main comparative evaluation on Ant in Fig.~\ref{fig:ant_schedule} (results on other environments are presented in App.~\ref{app:more_schedule}).
We find consistent trends across environments such as the decrease of learning rate and increase in batch sizes over time, consistent to common practices in both \gls{RL}~\citep{Engstrom2020Implementation} and supervised learning, but crucially \our discovers the same \emph{without any pre-defined schedule}.
We note, however, that the exact rate at which the learning rate decreases and batch size increase differs across different environments -- for example, in Ant we find that the learning rate quickly drops from a relatively large value to almost the smallest possible $10^{-4}$, whereas in UR5e, the schedule is much less aggressive. This suggests that the optimal schedule is dependent on the exact environment, and a uniform, manually-defined schedule as in~\citet{Engstrom2020Implementation} may not be optimal. We demonstrate this empirically in App. \ref{app:comp_rs}, where we compare against \gls{RS} but with the learning rate following a manually-defined cosine annealing schedule.
We also find that different networks are favored at different stages of training, but the exact patterns differ across environments: for Ant (Fig.~\ref{fig:ant_schedule}), we find that larger networks are preferred towards the end of training, with the policy and value network widths increasing over time:
Prior work has shown that larger networks like those we automatically find towards the end of training \emph{can be notoriously unstable and difficult to train from scratch}~\citep{czarnecki2018mixmatch, ota2021training}, which further supports our use of generational training to facilitate this.

\textbf{Ablation Studies.}
\our improves on existing methods by using local \gls{TR}-based \gls{BO} (§\ref{subsec:bo}) and \gls{NAS} \& distillation (§\ref{subsec:distillation}).
We conduct an ablation study by removing either or both components in \cref{tab:ablation} (a comparison between the training trajectories in Fig. \ref{fig:nonas_fig} may be found in App. \ref{app:ablation}), where ``No NAS'' does not search architectures or distill but uses the default \textsc{Brax} architectures, and ``No TR/NAS'' further only uses a vanilla \gls{GP} surrogate and is identical to \gls{PB2}. 
We find the tailored \gls{BO} agent in §\ref{subsec:bo} improves performance across the board.
On the importance of \gls{NAS} \& distillation, in all environments except for Humanoid and Reacher, \our matches or outperforms ``No NAS'', despite $\mathcal{J}$ being a more complicated search space and the ``No NAS'' baseline is conditioned on strongly-performing default architectures.
We also see a particularly large gain for HalfCheetah and Fetch when we include architectures, demonstrating the effectiveness of the generational training and \gls{NAS} in our approach.
We include additional ablation studies in App. \ref{app:ablation}.

%% file: 07_conclusion.tex
\section{Related Work}

\textbf{On-the-fly Hyperparameter Tuning.}
Our work improves on previous \gls{PBT}~\citep{pb2, PBT, pbtbt} style methods; in particular, we build upon \citet{parkerholder2021tuning}, using a more scalable \gls{BO} step, and adding architecture search with generational learning.
\citet{fire_pbt} introduce an approach for increasing diversity in the weight space for \gls{PBT}, orthogonal to our work.
There have also been non-population-based methods for dynamic hyperparameter optimization, using bandits \citep{agent57, top, dvd, rp1,gpbo_tvo}, gradients \citep{paul2019fast, metagradients, stac, flennerhag2022bootstrapped} or Evolution \citep{online_hpo_evolutionary_arvix_20} which mostly do not search over architectures.
A notable exception is Sample-efficient Automated Deep Learning (SEARL) \citep{franke2020sample}, which adapts architectures within a \gls{PBT} framework.
However, SEARL is designed for off-policy \gls{RL} and thus especially shows the benefit of shared replay buffers for efficiency, whereas our method is general-purpose.

\textbf{Architecture Search.}
In \gls{RL}, \citet{czarnecki2018mixmatch} showed increasing agent complexity over time could be effective, albeit with a pre-defined schedule.
\citet{rl_darts} showed that DARTS \citep{liu2018darts} could be effective in \gls{RL}, finding high performing architectures on the Procgen benchmark.
Auto-Agent-Distiller~\citep{fu2020autoagentdistiller} deals with the problem of finding optimal architectures for compressing the model size of \gls{RL} agents, and also find that using distillation between the teacher and student networks improves stability of \gls{NAS} in \gls{RL}.
On the other hand, \gls{BO} has been used as a powerful tool for searching over large architecture spaces~\citep{casmopolitan, kandasamy2018neural, ru2020interpretable, white2019bananas,nguyen2021optimal, wan2022approximate}.
Conversely, we only consider simple \gls{MLP}s and the use of spectral normalization.
There has been initial effort~\citep{izquierdo2021bag} combining \gls{NAS} and hyperparameter optimization in sequential settings, which is distinct to our on-the-fly approach.

\textbf{Generational Training and Distillation.}
\cite{openendedlearningteam2021openended} recently introduced generational training, using policy distillation to transfer knowledge between generations, accelerating training.
Our method is based on this idea, with changing generations.
The use of distillation is further supported by \citet{igl2021transient} who recently used this successfully to adapt to non-stationarities in reinforcement learning, however keeping hyperparameters and architectures fixed.

\section{Conclusion \& Discussion}
\label{sec:conclusion}
In this paper, we propose \our: a new algorithm that significantly increases the capabilities of \gls{PBT} methods for \gls{RL}.
Using recent advances in Bayesian Optimization, \our is capable of searching over drastically larger search spaces than previous methods.
Furthermore, inspired by recent advances in generational learning, we show it is also possible to efficiently learn architectures on the fly as part of a unified algorithm.
The resulting method leads to significant performance gains across the entire  \textsc{Brax} environment suite, achieving high performance even in previously untested environments.
We believe \our is a significant step towards truly environment agnostic RL algorithms, while also offering a path towards open-ended learning where agents never stop tuning themselves and continuously expand their capabilities over time.

%% file: 08_limitations.tex
\textbf{Limitations \& Future Work.}
We note that while our method shows a significant boost in performance by including architectures for most environments, in some environments, such as Humanoid, we achieve better results without architecture search (Table \ref{tab:ablation}).
We hypothesize this is due to the complexity of the environment and an increased sensitivity to the network architecture.
Furthermore, while we provide a theoretical guarantee for our method in \cref{thr:gconverge-mix} for searching purely over architectures, no such guarantees exist when we transfer between architectures across generations.
Indeed, we occasionally see poor architectures being selected, which are then discarded during truncation selection.
Therefore, an immediate future direction is to address these issues and to improve the architecture selection process.
Another limitation is that while all \gls{RL}-related hyperparameters are included in the search space, certain hyperparameters of \our could also be automatically searched for, including but not limited to distillation hyperparameters, which are currently fixed, and \gls{PBT} parameters such as $t_{\mathrm{ready}}$, which could allow us to avoid myopic and greedy behavior.
Beyond these limitations, our algorithm readily transfers to other \gls{RL} problems with high-dimensional mixed spaces, and thus we would readily accommodate more complicated architecture search spaces (e.g. vision-based environments) and incorporate environment parameters~\citep{DBLP:journals/corr/PaulCOW16} into the search space to generalize to new tasks.

\textbf{Broader Impact.}
We open-source our code so that practitioners in the field can accelerate their own deployment of \gls{RL} systems.
However, in doing so, we should be wary of the risk of also improving malicious use of \gls{RL}; in particular, down-stream applications which could have an impact on people's security and privacy.
To mitigate these risks, we encourage research on \gls{RL} governance and safe \gls{RL}.
As a general purpose framework for improving any \gls{RL} algorithm, our method should be part of that conversation.

%% file: 10_ack.tex
\begin{acknowledgements}
XW and BR are supported by the Clarendon Scholarship at the University of Oxford.
CL is funded by the Engineering and Physical Sciences Research Council (EPSRC). 
PB is funded through the Willowgrove Studentship.
The authors would like to thank Yee Whye Teh for detailed feedback on an early draft and the anonymous AutoML conference reviewers \& the area chair for their constructive comments which helped improve the paper.
\end{acknowledgements}

%% file: 11_appendix.tex
\newpage

\section*{\LARGE \centering Supplementary Material}
\label{sec:appendix}

\section{Primer on GPs and BO}
\label{app:primer}

\paragraph{Gaussian Processes}
In Bayesian Optimization (\gls{BO}), Gaussian Processes, or \gls{GP}s, act as \emph{surrogate models} for a black-box function $f$ which takes an input  $\bz$ (in our case, the hyperparameters and/or the architecture parameters) and returns an output $y=f(\bz) + \epsilon$ where $\epsilon \sim \mathcal{N}(0,\sigma^2)$.
A \gls{GP} defines a probability distribution over functions $f$ under the assumption that any finite subset $\lbrace (\bz_i, f(\bz_i) \rbrace$ follows a normal distribution \citep{RasmussenGP}.
Formally, a \gls{GP} is denoted as $f(\bz)\sim \text{GP}\left(m\left(\bz\right),k\left(\bz,\bz'\right)\right)$, where $m\left(\bz\right)$ and $k\left(\bz,\bz'\right)$ are called the mean and covariance functions respectively, i.e. $m(\bz)=\mathbb{E}\left[f\left(\bz\right)\right]$ and $k(\bz,\bz')=\mathbb{E}\left[(f\left(\bz\right)-m\left(\bz\right))(f\left(\bz'\right)-m\left(\bz'\right))^{T}\right]$.
The covariance function (kernel) $k(\bz,\bz')$ can be thought of as a similarity measure relating $f(\bz)$ and $f(\bz')$.
There have been various proposed kernels which encode different prior beliefs about the function $f(\bz)$ \citep{RasmussenGP}. 

If we assume a zero mean prior $m(\bz)=0$, to predict $f_{*}=f\left(\bz_{*}\right)$ at a new data point $\bz_{*}$, we have,
\begin{align}
\left[\begin{array}{c}
\boldsymbol{f}\\
f_{*}
\end{array}\right] & \sim\mathcal{N}\left(0,\left[\begin{array}{cc}
\bK & \bk_{*}^{T}\\
\bk_{*} & k_{**}
\end{array}\right]\right),\label{eq:p(f|f*)}
\end{align}
 where $k_{**}=k\left(\bz_{*},\bz_{*}\right)$, $\bk_{*}=[k\left(\bz_{*},\bz_{i}\right)]_{i\le t}$, $t$ is the number of observed points for the \gls{GP},
and $\bK=\left[k\left(\bz_{i},\bz_{j}\right)\right]_{i,j\le t}$.
We denote our observations as $\{\bz_1, f_1\}, \{\bz_2, f_2\}, ..., \{\bz_t, f_t\}$ and collect all past return observations as $\mathbf{f}_t = [f_1, ..., f_t]^{\top}$.
Then, we may combine Eq. (\ref{eq:p(f|f*)}) with the fact that $p\left(f_{*}\mid\boldsymbol{f}\right)$ follows a univariate Gaussian distribution $\mathcal{N}\left(\mu\left(\bz_{*}\right),\sigma^{2}\left(\bz_{*}\right)\right)$. Given a new configuration $\bz'$, the \gls{GP} posterior mean and variance at $\bz'$ may be computed as:
\begin{equation}
\label{eqn:mu}
    \mu_t(\bz') \coloneqq \textbf{k}_t(\bz')^T(\textbf{K}_t + \sigma^2\textbf{I})^{-1}\mathbf{f}_t
\end{equation}
\begin{equation}
\label{eqn:sig}
    \sigma_t^2(\bz') \coloneqq k(\bz', \bz') - \textbf{k}_t(\bz')^T(\textbf{K}_t +\sigma^2\textbf{I})^{-1}\textbf{k}_t(\bz'),
\end{equation}
where $\mathbf{K}_t \coloneqq \{k(z_i, z_j)\}_{i, j=1}^t$ and $\mathbf{k}_t \coloneqq \{k(z_i, z'_t)\}_{i=1}^t$.

\paragraph{Bayesian Optimization} 

Bayesian optimization (\gls{BO}) is a powerful sequential approach to find the global optimum of an expensive black-box function $f(\mathbf{z})$ without making use of derivatives.
First, a surrogate model (in our case, a \gls{GP} as discussed above) is learned from the current observed data $\mathcal{D}_t= \lbrace \mathbf{z}_i, y_i \rbrace_{i=1}^t$ to approximate the behavior of $f(\mathbf{z})$.
Second, an \emph{acquisition function} is derived from the surrogate model to select new data points that maximizes information about the global optimum -- a common acquisition function that we use in our paper is the Upper Confidence Bound (\gls{UCB}) \citep{Srinivas_2010Gaussian} criterion which balances exploitation and exploration.
Specifically, the \gls{UCB} on a new, unobserved point $\bz'$ is given by:
\begin{equation}
    \mathrm{UCB}(\bz') = \mu_t(\bz') + \sqrt{\beta_t}\sigma_t(\bz'),
\end{equation}
where $\mu_t$ and $\sigma_t$ are the posterior mean and standard deviation given in Eq. \ref{eqn:mu} above and $\beta_t > 0$ is a trade-off parameter between mean and variance.
At each \gls{BO} iteration, we find a batch of samples that sequentially maximizes the acquisition function above. 
The process is conducted iteratively until the evaluation budget is depleted, and the global optimum is estimated based on all the sampled data.
In-depth discussions about \gls{BO} beyond this brief overview can be found in recent surveys \citep{Brochu_2010Tutorial,Shahriari_2016Taking,frazier2018tutorial}.

\section{Bayesian Optimization for PBT}
\label{app:bo_pbt}
In this section, we provide specific details for the modifications to \textsc{Casmopolitan} to make it amenable for our setup which consists of non-stationary reward and a mixed, high-dimensional search space.

\subsection{Kernel Design}
We use the following time-varying kernel~\citep{parkerholder2021tuning, bogunovic2016time} to measure the spatiotemporal distance between a pair of configuration vectors $\{\bz, \bz'\}$ with continuous, ordinal and/or categorical dimensions, and whose rewards are observed at timesteps $\{i, j\}$. For the most general case where we model all three types of variables, we have the following kernel function:
\begin{equation}
\label{eq:kernel}
    k(\bz, \bz', i, j) = \frac{1}{2}\Big(\big(k_x(\bx, \bx') + k_h(\bh, \bh')\big) + \big(k_x(\bx, \bx')k_h(\bh, \bh')\big)\Big) \Big( (1-\omega)^{|i-j|/2} \Big)
\end{equation}
where $\bx$ denotes the continuous \emph{and ordinal} dimensions and $\bh$ denotes the categorical dimensions of the configuration vector $\bz$, respectively, $k_x(\cdot, \cdot)$ is the kernel for continuous and ordinal inputs (by default Mat\'ern 5/2), $k_{(h)}(\cdot, \cdot)$ is the kernel for the categorical dimensions (by default the exponentiated overlap kernel in \citet{casmopolitan}) and $\omega \in [0, 1]$ controls how quickly old data is decayed and is learned jointly during optimization of the GP log-likelihood.
When the search space only contains continuous/ordinal variables, we simply have $k(\bz, \bz',i,j) = k_z(\bx, \bx') (1-\omega)^{|i-j|/2}$, and a similar simplification holds if the search space only contains categorical variables.
We improve on~\citet{parkerholder2021tuning} by directly supporting ordinal variables such as integers (for e.g. batch size) and selecting them alongside categorical variables using \emph{interleaved acquisition optimization} as opposed to time-varying bandits which scales poorly to large discrete spaces. 

\subsection{Proposing New Configurations}
\label{appsubsec:propose}
As discussed in App. \ref{app:primer}, a \gls{BO} agent selects new configuration(s) by selecting those which maximize the acquisition function (in this case, the \gls{UCB} acquisition function).
This is typically achieved via off-the-shelf first-order optimizers, which is challenging in a mixed-input space as the discrete (ordinal and categorical) variables lack gradients and na\"ively casting them into continuous variables yields invalid solutions which require rounding.
To address this issue, \citet{parkerholder2021tuning} select $\bh$ first via time-varying bandits (using the proposed TV.EXP3.M algorithm) and then select $\bx$ by optimizing the BO acquisition function, \emph{conditioned on} the chosen $\bh$.
This method scales poorly to spaces with a large number of categorical choices, as bandit problems generally require pulling each arm at least once.
Instead, we develop upon \emph{interleaved acquisition optimization} introduced in \citet{casmopolitan} which unifies all variables under a single GP, and alternates between optimization of the continuous and discrete variables:

\begin{algorithm}[H]
\begin{footnotesize}
    \caption{Interleaved optimization of $\mathrm{acq}(\bz)$}
    \label{alg:interleaved_optimisation}
	\begin{algorithmic}[1]
	\WHILE{not converged}
	\STATE \textbf{Continuous:} Do a single step of gradient descent on the continuous dimensions.
	\STATE \textbf{Ordinal and Categorical}: Conditioned on the new continuous values, do a single step of local search: randomly select an ordinal/categorical variable and choose a different (categorical), or an adjacent (ordinal) value, if the new value leads to an improvement in $\mathrm{acq}(\cdot)$.
	\ENDWHILE
\end{algorithmic}
\end{footnotesize}
\end{algorithm}

Compared to the approach in \citet{casmopolitan}, we include ordinal variables, which are optimized alongside the categorical variables via local search during acquisition optimization but are treated like continuous variables by the kernel.
During acquisition, we define adjacent ordinals to be the neighboring values. For example, for an integer variable with a valid range $[1, 5]$ and current value $3$, its neighboring values are $2$ and $4$.
This allows us to exploit the natural ordering for ordinal variables whilst still ensuring that suggested configurations remain local and only explore \emph{valid} neighboring solutions.

\subsection{Suggesting New Architectures}
\label{appsubsec:new_archs}
At the start of each generation for the full \our method, we have to suggest a pool of new architectures. For the first generation, we simply use random sampling across the joint space $\mathcal{J}$ to fill up the initial population. For subsequent generations, we use a combination of \gls{BO} and random sampling to both leverage information already gained from the architectures and allow sufficient exploration. For the \gls{BO}, at the start of the $i$-th generation, we first fit a \gls{GP} model solely in the architecture space $\mathcal{Y}$, by using the architectures from the $i-1$-th generation as the training data. Since these network architectures are trained with different hyperparameters during the generation, we use the \emph{best return} achieved on each of these architectures as the training targets. We then run \gls{BO} on this \gls{GP} to obtain the suggestions for new architectures for the subsequent generation. In practice, to avoid occasional instability in the distillation process, we find it beneficial to select a number of architectures larger than $B$: we then start the distillation for all the agents, but use successive halving \citep{pmlr-v28-karnin13} such that only $B$ agents survive and are distilled for the full budget allocated. By doing so, we trade a modest increase in training steps for greatly improved stability in distillation. 

\subsection{Details on Trust Regions}
\label{appsubsec:restart_global}

To define trust regions for our time-varying objective, we again consider the most general case where the search space contains both categorical and continuous/ordinal dimensions. Given the configuration $\bz^*_{t} = [\bh^*_t, \bx^*_t] = \arg \max_{\bz_t} (f_t)$ with the best return
at time $t$, we may define the trust region centered around $\bz^*_{t}$:
\begin{equation}
    \mathrm{TR}(\bz^*_T) = \begin{cases}
    \big\{ \bh \mid \frac{1}{d_h}\sum_{i=1}^{d_h} \delta(h_i, h^*_i) \leq L_h \big\} & \text{for categorical $\bh^*_T = \{h^*_i\}_{i=1}^{d_h}$ } \\
    \big\{ \bx \mid |x_i - x^*_i | < \frac{\Tilde{\ell}_i}{\prod_{i=1}^{d_x} \Tilde{\ell}_i^{\frac{1}{d_x}}} L_x, 0 \leq x_i \leq 1 \big\}  & \text{for continuous or ordinal $\bx^*_T = \{x^*_i\}_{i=1}^{d_x}$} , \\
    \end{cases}
\end{equation}
where $\delta(\cdot, \cdot)$ is the Kronecker delta function, $L_h \in [0, 1]$ is the trust region radius defined in terms of normalized Hamming distance over the categorical variables, $L_x$ is the trust region radius defining a hyperrectangle over the continuous and ordinal variables, and $\{\Tilde{\ell_i} = \frac{\ell_i}{\frac{1}{d_x}\sum_{i=1}^{d_x}\ell_i} \}_{i=1}^{d_x}$ are the normalized lengthscales $\{\ell_i\}$ learned by the GP surrogate over the continuous/ordinal dimensions.
This means that the more sensitive hyperparameters, i.e. those with smaller learned lengthscales, will automatically be assigned smaller trust region radii.

For the restart of trust regions when either or both of the trust regions defined fall below some pre-defined threshold, we adapt the \gls{UCB}-based criterion proposed in~\citet{casmopolitan} to the time-varying setting to re-initialize the population when a restart is triggered.
For the $i$-th restart, we consider a global, auxiliary GP model trained on a subset of observed configurations and returns $D^*_{i-1} = \{\bz^*_j, f^*_j\}_{j=1}^{i}$ and denote $\mu_g(\bz; D^*_{i-1})$ and $\sigma_g^2(\bz; D^*_{i-1})$ as the posterior mean and variance of the auxiliary GP. The new trust region center is given by the configuration $\bz_i^{(0)}$ that maximizes the UCB score: $\bz_i^{(0)} = \arg \max_{\bz \in \mathcal{Z}} \mu_g(\bz; D^*_{i-1}) + \sqrt{\beta_i} \sigma_g(\bz; D^*_{i-1})$ where $\beta_i$ is the UCB trade-off parameter.
In the original \textsc{Casmopolitan}, $D^*$ consists of the best configurations in all previous restarts $1, ..., i-1$, which is invalid for the time-varying setting. Instead, we construct $D^*_{i-1}$ using the following:
\begin{equation}
    D^*_{i-1} = \{\bz_j^*, \mu_T(\bz_j^*)\}_{j=1}^{i=1} \text{ where } \bz_j^* = \arg \max_{\bz_j \in \mathcal{D}_j } \mu_T(\bz_{j}),
\end{equation}
where $\mathcal{D}_j$ denotes the set of previous configurations evaluated during the $j$-th restart and $\mu_T(\cdot)$ denotes the posterior mean of the time-varying GP surrogate \emph{at the present timestep $t=T$}. Thus, instead of simply selecting the configurations of each restart that \emph{led} to the highest observed return, we select the configurations that \emph{would have led to the highest return if they were evaluated now}, according to the GP surrogate. Such a configuration preserves the convergence property of \our (without distillation and architecture search) shown in \cref{thr:gconverge-mix} and proven below in App. \ref{app:proofs}.

\section{Theoretical Guarantees}
\label{app:proofs}

\subsection{Bound on the Maximum Information Gain}

We start by deriving the maximum information gain, which extends the result presented in \citet{casmopolitan} for the time-varying setting. Note that this result is defined over the number of local restarts $I$. 

\begin{theorem} \label{thr:tv_kernel_mig}
Let $\gamma(I;k;V) := \max_{A \subseteq V, \vert A \vert \leq I} \frac{1}{2} \log \vert \idenmat + \sigma^{-2} \lbrack k(\mathbf{v}, \mathbf{v}') \rbrack_{\mathbf{v}, \mathbf{v}' \in A} \vert$ be the maximum information gain achieved by sampling $I$ points in a \gls{GP} defined over a set $V$ with a kernel  $k$. Denote the constant $\eta := \prod_{j=1}^{d_h} n_j$. Then we have, for the time-varying mixed kernel $k$, 
\begin{equation}
\gamma(I;k;[\mathcal{H}, \mathcal{X}]) \lessapprox
 \frac{I}{\tilde{N}}\left(\lambda \eta\gamma(I;k_x;\mathcal{X}) + \\
(\eta-2\lambda)\log I+\sigma_{f}^{-2}\tilde{N}^{3}\omega\right)    
\end{equation}
where the time steps $\left\{ 1,...,I\right\}$ are split into  into $I/\tilde{N}$
blocks of length $\tilde{N}$, such that the function
$f_{t}$ does not vary significantly within each block.
\end{theorem}

\begin{proof}
Following the proof used in \cite{bogunovic2016time}), we split the time steps $\left\{ 1,...,I\right\} $ into $I/\tilde{N}$
blocks of length $\tilde{N}$, such that within each block the function
$f_{i}$ does not vary significantly. Then, we have that the maximum information gain of the time-varying kernel \cite{bogunovic2016time}) is bounded by 
\begin{align*}
\gamma_{I} & \le\left(\frac{I}{\tilde{N}}+1\right)\left( \tilde{ \gamma}_{\tilde{N}}+\sigma_{f}^{-2}\tilde{N}^{3}\omega\right)
\end{align*}
where $\omega \in [0,1]$ is the forgetting-remembering trade-off parameter, and we consider the kernel for time $1-k_{time}(t,t')\le\omega\left|t-t'\right|$.
We denote $\tilde{ \gamma}_{\tilde{N}}$ as the maximum information gain for the time-invariant kernel counterpart in each block length of $\tilde{N}$.

Next, by using the bounds for the (time-invariant) mixed kernel in \cite{casmopolitan} that $\tilde{ \gamma}_{\tilde{N}} \le \mathcal{O} \big( (\lambda \eta + 1 - \lambda)\gamma(I;k_x;\mathcal{X}) + (\eta+2-2\lambda)\log I \big)$, we get the new time-varying bound $\gamma(I;k;[\mathcal{H}, \mathcal{X}]) \lessapprox
 \frac{I}{\tilde{N}}\left(\lambda \eta\gamma(I;k_x;\mathcal{X}) + \\
(\eta-2\lambda)\log I+\sigma_{f}^{-2}\tilde{N}^{3}\omega\right)$ where we have suppressed the constant term for simplicity.
\end{proof}

\subsection{Proof of Local Convergence in Each Trust Region}

\begin{assumption} \label{assu:f-bounded}
The time-varying objective function $f_t(\mathbf{z})$ is bounded in $[\mathcal{H}, \mathcal{X}]$, i.e. $\exists F_l, F_u \in \mathbb{R} : \forall \mathbf{z} \in [\mathcal{H}, \mathcal{X}]$, $F_l \leq  f_t(\mathbf{z}) \leq F_u, \forall t \in [1,...,T]$.
\end{assumption}

\begin{assumption} \label{assu:gp-approx} Let us denote $L^h_{\min}$, $L^x_{\min}$ and $L_0^h, L_0^x$ be the minimum and initial \gls{TR} lengths for the categorical and continuous variables, respectively. Let us also denote $\alpha_s$ as the shrinking rate of the \gls{TR}s. The local \gls{GP} approximates $f_t, \forall t \le T$ accurately within any \gls{TR} with length $L^x \leq \max \big(L^x_{\min}/\alpha_s, L_0^x (\lceil (L^h_{\min}+1)/\alpha_s \rceil -1)/L_0^h \big)$ and $L^h \leq \max \big(\lceil (L^h_{\min}+1)/\alpha_s \rceil - 1, \lceil L_0^h L^x_{\min}/ (\alpha_s L_0^x) \rceil \big)$.
\end{assumption}

\begin{theorem} \label{thr:local-converge}
Given Assumptions \ref{assu:f-bounded} \& \ref{assu:gp-approx}, after a restart, \ourNAS converges to a local maxima after a finite number of iterations or converges to the global maximum.
\end{theorem}

\begin{proof}
We may apply the same proof by contradiction used in \cite{casmopolitan} for our time-varying setting, given the assumptions \ref{assu:f-bounded} and \ref{assu:gp-approx}. For completeness, we summarize it below.

We show that our algorithm converges to (1) to a global maximum of $f$ (if does not terminate after a finite number of iterations) or (2) a local maxima of $f$ (if terminated after a finite number of iterations). 

Case 1: when $t\rightarrow \infty$ and the \gls{TR} lengths $L^h$ and $L^x$ have not shrunk below $L^h_{\min}$ and $L^x_{\min}$.
From the algorithm description, the \gls{TR} is shrunk after \texttt{fail\_tol} consecutive failures. Thus, if after $N_{\min} = \texttt{fail\_tol} \times m$ iterations where $m = \max(\lceil \log_{\alpha_e}(L^h_0/L^h_{\min}) \rceil, \lceil \log_{\alpha_e}(L^x_0/L^x_{\min}) \rceil)$, there is no success, \ourNAS terminates. This means, for case (1) to occur, \ourNAS needs to have at least one improvement per $N_{\min}$ iterations. Let consider the increasing series $\lbrace f(\mathbf{z}^k) \rbrace_{k=1}^{\infty}$ where $f(\mathbf{z}^k)= \max_{t=(k-1)N_{\min}+1,\dots, kN_{\min}} \lbrace f(\mathbf{z}_t) \rbrace$ and $f(\mathbf{z}_i)$ is the function value at iteration $t$. Thus, using the monotone convergence theorem \citep{bibby1974axiomatisations}, this series converges to the global maximum of the objective function $f$ given that $f(\mathbf{z})$ is bounded (Assumption \ref{assu:f-bounded}). 

Case 2: when \ourNAS terminates after a finite number of iterations, \ourNAS converges to a local maxima of  $f(\mathbf{z})$ given Assumption \ref{assu:gp-approx}. Let us remind that \ourNAS terminates when either the continuous \gls{TR} length $\leq L^x_{\min}$ or the categorical \gls{TR} length $\leq L^h_{\min}$. 

Let $L_s$ be the largest \gls{TR} length that after being shrunk, the algorithm terminates, i.e., $\lfloor \alpha_s L_s \rfloor \leq L^h_{\min}$.\footnote{The operator $\lfloor . \rfloor$ denotes the floor function} Due to $\lfloor \alpha_s L_s \rfloor \leq \alpha_s L_s < \lfloor \alpha_s L_s \rfloor +1$, we have $ L_s < (L^h_{\min}+1)/\alpha_s$. Because $L_s$ is an integer, we finally have $L_s \leq \lceil (L^h_{\min}+1)/\alpha_s \rceil - 1$. This means that $L_s= \lceil (L^h_{\min}+1)/\alpha_s \rceil - 1$ is the largest \gls{TR} length that after being shrunk, the algorithm terminates. We may apply a similar argument for the largest \gls{TR} length (before terminating) for the continuous  $L^x_{\min}/\alpha_s$.

In our mixed space setting, we have two separate trust regions for categorical and continuous variables. When one of the \gls{TR} reaches its terminating threshold ($L^x_{\min}/\alpha_s$ or $\lceil (L^h_{\min}+1)/\alpha_s \rceil -1$), the length of the other one is ($\lceil L_0^h L^x_{\min}/ (\alpha_s L_0^x) \rceil \big)$ or $L_0^x (\lceil (L^h_{\min}+1)/\alpha_s \rceil -1)/L_0^h$).
Based on \cref{assu:gp-approx}, the \gls{GP} can accurately fit a \gls{TR} with continuous length $L^x \leq \max \big(L^x_{\min}/\alpha_s, L_0^x (\lceil (L^h_{\min}+1)/\alpha_s \rceil -1)/L_0^h \big)$ and $L^h \leq \max \big(\lceil (L^h_{\min}+1)/\alpha_s \rceil - 1, \lceil L_0^h L^x_{\min}/ (\alpha_s L_0^x) \rceil \big)$.
Thus, if the current \gls{TR} center is not a local maxima, \ourNAS can find a new data point whose function value is larger than the function value of current \gls{TR} center.
This process occurs iteratively until a local maxima is reached, and \ourNAS terminates.
\end{proof}

\subsection{Proof of Theorem \ref{thr:gconverge-mix}}

\begin{proof}
\label{sec:proof-mixconverge}
Under the time-varying setting, at the $i$-th restart, we first fit the global time-varying \gls{GP} model on a subset of data $D^*_{i-1}= \lbrace \mathbf{z}^*_j, f(\mathbf{z}^*_j) \rbrace_{j=1}^{i-1}$, where $\mathbf{z}^*_j$ is the local maxima found after the $j$-th restart, or, a random data point, if the found local maxima after the $j$-th restart is same as in the previous restart.

Let $\bz_{i}^{**}=\arg\max_{\forall\bz \in [\mathcal{H}, \mathcal{X}]}f_{t}(\bz)$ \footnote{Notationally, at the $i$-th restart, $\bz_{i}^{**}$ is the global optimum location while  $\bz_{i}^{*}$ is the local maxima found by \ourNAS.} be the global optimum location at time step $i$. Let $\mu_{gl}(\mathbf{z}; D^*_{i-1})$ and $\sigma^2_{gl}(\mathbf{z}; D^*_{i-1})$ be the posterior mean and variance of the global \gls{GP} learned from $D^*_{i-1}$. Then, at the $i$-th restart, we select the following location $\mathbf{z}^{(0)}_i$ as the initial centre of the new \gls{TR}:
\begin{align} \nonumber
    \mathbf{z}^{(0)}_i= \arg \max_{\mathbf{z} \in [\mathcal{H}, \mathcal{X}]} \mu_{gl}(\mathbf{z}; D^*_{i-1}) + \sqrt{\beta_i} \sigma_{gl}(\mathbf{z}; D^*_{i-1}),
\end{align}
where $\beta_i$ is the trade-off parameter in \gls{PB2} \citep{pb2}. 

We follow \cite{casmopolitan} to assume that at the $i$-th restart, there exists a function $g_i(\mathbf{z})$: (a) lies in the RKHS $\mathcal{G}_k(\lbrack \mathcal{H}, \mathcal{X} \rbrack)$ and $\Vert g_i \Vert_k^2 \leq B$, (b) shares the same global maximum $\mathbf{z}^*$ with $f$, and, (c) passes through all the local maxima of $f$ and any data point $\mathbf{z'}$ in $\mathcal{D}^*_{i-1} \cup \lbrace \mathbf{z}^{(0)}_i \rbrace$ which are not local maxima (i.e. $g_i(\mathbf{z'}) = f(\mathbf{z'}), \forall \mathbf{z'} \in D^*_{i-1} \cup \lbrace \mathbf{z}^{(0)}_i \rbrace$). In other words, the function $g_i(\mathbf{z})$ is a function that passes through the maxima of $f$ whilst lying in the RKHS $\mathcal{G}_k(\lbrack \mathcal{H}, \mathcal{X} \rbrack)$ and satisfying $\Vert g_i \Vert_k^2 \leq B$. %

Using $\beta_i$ defined in Theorem 2 in \citet{Srinivas_2010Gaussian} for function $g_i$, $\forall i$, $\forall z \in \lbrack \mathcal{H}, \mathcal{X} \rbrack$, we have,
\begin{equation} \label{eq:sriniva-tr6}
\begin{aligned}
    & \text{Pr} \lbrace \vert \mu_{gl}(\mathbf{z}; D_{i-1}^*) - g_i(\mathbf{z}) \vert \leq \sqrt{\beta_i} \sigma_{gl}(\mathbf{z}; D_{i-1}^*) \vert  \rbrace \geq 1-\zeta.
\end{aligned}
\end{equation}
In particular, with probability $1-\zeta$, we have that,
\begin{equation}
\begin{aligned}
 \mu_{gl}(\mathbf{z}^{(0)}_i; D^*_{i-1}) + \sqrt{\beta_i} \sigma_{gl}(\mathbf{z}^{(0)}_i; D^*_{i-1}) 
 \geq \mu_{gl}(\mathbf{z}_i^{**}; D^*_{i-1}) + \sqrt{\beta_i} \sigma_{gl}(\mathbf{z}_i^{**}; D^*_{i-1}) \geq g_i(\mathbf{z}_i^{**}).
\end{aligned}
\end{equation}
Thus, $\forall i$, with probability $1-\zeta$ we have
\begin{equation} \nonumber
\begin{aligned}
& g_i(\mathbf{z}_i^{**}) - g_i(\mathbf{z}^{(0)}_i) \leq \mu_{gl}(\mathbf{z}^{(0)}_i; D^*_{i-1}) + \sqrt{\beta_i} \sigma_{gl}(\mathbf{z}^{(0)}_i; D^*_{i-1}) - g_i(\mathbf{z}^{(0)}_i) 
 \leq 2\sqrt{\beta_i} \sigma_{gl}(\mathbf{z}^{(0)}_i; D^*_{i-1}).
\end{aligned}
\end{equation}
Since $g_i(\mathbf{z}^{(0)}_i) = f(\mathbf{z}^{(0)}_i)$, and $g_i(\mathbf{z}_i^{**})=f(\mathbf{z}_i^{**})$, hence, $f_i(\mathbf{z}_i^{**}) - f(\mathbf{z}^{(0)}_i) \leq 2\sqrt{\beta_i} \sigma_{gl}(\mathbf{z}^{(0)}_i; D^*_{i-1})$ with probability $1-\zeta$. With $\mathbf{z}_i^*$ as the local maxima found by \ourNAS at the $i$-th restart. As $f(\mathbf{z}^{(0)}_i) \leq f(\mathbf{z}_i^*)$, we have,
\begin{align} 
    f_i(\mathbf{z}_i^{**}) - f_i(\mathbf{z}_i^*) \leq 2\sqrt{\beta_i} \sigma_{gl}(\mathbf{z}^{(0)}_i; D^*_{i-1}).  \label{eq:bounded_fz**_fz*}
\end{align}
Let $\bz_{i,b}$ be the point chosen by our algorithm at iteration $i$ and batch element $b$, we follow \cite{pb2} to define the time-varying instantaneous regret as $r_{i,b}=f_{i}(\bz_{i}^{**})-f_{i}(\bz_{i,b})$. 
Then, the time-varying batch instantaneous regret over $B$ points is as follows
\begin{align}
r_{i}^{B} & =\min_{b\le B}r_{i,b}=\min_{b\le B}f_{i}(\bz_{i}^{**})-f_{i}(\bz_{i,b}),\forall b\le B
\end{align}
Using \cref{eq:bounded_fz**_fz*} and Theorem 2 in \cite{pb2}, we bound the cumulative batch regret over $I$ restarts and $B$ parallel agents
\begin{align}
R_{IB} & = \sum_{i=1}^{I}r_{i}^{B}  
 \le\sqrt{\frac{C_1 I\beta_{I}}{B}
 \gamma(IB;k;[\mathcal{H}, \mathcal{X}])}+2\label{eq:R_T_last_step}
\end{align}
where $C_{1}=32/\log(1+\sigma_{f}^{2})$, $\beta_I$ is the explore-exploit hyperparameter defined in Theorem 2 in \cite{pb2} and $\gamma(IB;k;[\mathcal{H}, \mathcal{X}]) \lessapprox
\frac{ IB}{\tilde{N}}\left(\lambda \eta\gamma(I;k_x;\mathcal{X}) + \\
(\eta-2\lambda)\log IB+\sigma_{f}^{-2}\tilde{N}^{3}\omega\right)$ is the maximum information gain defined over the mixed space of categorical and continuous $[\mathcal{H}, \mathcal{X}]$ in the time-varying setting defined in Theorem \ref{thr:tv_kernel_mig}.

\end{proof}

We note that given \cref{thr:gconverge-mix}, if we use the squared exponential kernel over the continuous variables, $\gamma(\tilde{N}B;k;\mathcal{X})=\mathcal{O}(\left[\log\tilde{N}B\right]^{d+1})$ \citep{Srinivas_2010Gaussian},
the bound becomes $R_{IB}\le\sqrt{\frac{C_{1}I^2\beta_{I}}{\tilde{N}}
 \left(\lambda \eta \left[\log\tilde{N}B\right]^{d+1} + \\
(\eta-2\lambda)\log IB+\sigma_{f}^{-2}\tilde{N}^{3}\omega\right) }+2$
where $\tilde{N}\le I$, $B\ll T$ and $\omega\in [0,1]$.

\section{Full PPO Hyperparameter Search Space}
\label{app:full_search_space}
We list the full search space for PPO in \cref{tab:ppo_hypers}.
The architecture and hyperparameters form the full 15-dimensional mixed search space. For methods that do not search in the architecture space (e.g., \gls{PBT}, \gls{PB2}, random search baselines in $\mathcal{Z}$, and the partial \our in Ablation Studies that uses §\ref{subsec:bo} only), the last 6 dimensions are fixed to the default architecture used in \textsc{Brax}: a policy network with 4 hidden layers each containing 32 neurons, and a value network with 5 hidden layers each containing 256 neurons. By default, spectral normalization is disabled in both networks.

\begin{table}[h]
\centering
\vspace{-4mm}
\caption{The hyperparameters for PPO form a 15-dimensional mixed search space.}
\vspace{-2mm}
\label{tab:ppo_hypers}
\begin{tabular}{@{}lll@{}}
\toprule
Hyperparameter                 & Type        & Range        \\ \midrule
learning rate                  & log-uniform & [1e-4, 1e-3] \\
discount factor ($\gamma$)     & uniform     & [0.9, 0.9999]\\
entropy coefficient (c)        & log-uniform & [1e-6, 1e-1] \\
unroll length                  & integer     & [5, 15]      \\
reward scaling                 & uniform     & [0.05, 20]   \\
batch size                     & integer (power of 2)   & [32, 1024]   \\
no. updates per epoch          & integer     & [2, 16]             \\
GAE parameter ($\lambda$)      & uniform     & [0.9, 1]             \\
clipping parameter ($\epsilon$)& uniform            & [0.1, 0.4]             \\
$\pi$ network width            & integer (power of 2) & [32, 256]             \\
$\pi$ network depth            & integer           & [1, 5]             \\
$\pi$ use spectral norm        & binary            & [True, False]             \\
$V$ network width              & integer (power of 2) & [32, 256]  \\
$V$ network depth              & integer            & [1, 5]             \\
$V$ use spectral norm          & binary            & [True, False]             \\ \bottomrule
\end{tabular}
\end{table}

\begin{table}[h]
\centering
\caption{Hyperparameters for \our.}
\vspace{-2mm}
\label{tab:bgpbt_hypers}
\begin{tabular}{@{}lll@{}}
\toprule
Hyperparameter                & Value  & Description \\ \midrule
$B$              & 8 & Population size (number of parallel agents)\\
$q$  & 12.5 & \% agents replaced each iteration ($q$) \\
$t_\textrm{max}$              & 150M & Total timesteps \\
$\alpha_{\textrm{RL}}$        & 1 & RL weight \\
$\alpha_{V}$                  & 0 & Value function weight \\
$\alpha_{\pi}$                & 5 & Policy weight \\
\bottomrule
\end{tabular}
\vspace{-2mm}
\end{table}
~
\begin{table}[h]
\centering
\vspace{-2mm}
\footnotesize{
\caption{Hyperparameters for \our inherited from \textsc{Casmopolitan}}
\vspace{-2mm}
\label{tab:bgpbt_hypers_casmo}
\begin{tabular}{lll}
\toprule
Hyperparameter                & Value & Description  \\
\midrule
\gls{TR} multiplier & 1.5 & multiplicative factor for each expansion/shrinking of the \gls{TR}. \\
\texttt{succ\_tol} & 3 & number of consecutive successes before expanding the \gls{TR} \\
\texttt{fail\_tol} & 10 & number of consecutive failures before shrinking the \gls{TR} \\
Min. continuous \gls{TR} radius & 0.15 & min. \gls{TR} of the continuous/ordinal variables before restarting \\
Min. categorical \gls{TR} radius & 0.1 & min. \gls{TR} of the categorical variables before restarting \\
Init. continuous \gls{TR} radius & 0.4 & initial \gls{TR} of the continuous/ordinal variables \\
Init. categorical \gls{TR} radius & 1 & initial \gls{TR} of the categorical variables \\
\bottomrule
\end{tabular}
}
\vspace{-4mm}
\end{table}

\section{Implementation Details}
\label{app:implementation_details}
We list the hyperparameters for our method \our in \cref{tab:bgpbt_hypers}.
Since \our uses the \textsc{Casmopolitan} \gls{BO} agent, it also inherits the default hyperparameters from \cite{casmopolitan} which are used in all our experiments (Table \ref{tab:bgpbt_hypers_casmo}).
We refer the readers to App. B.5 of \cite{casmopolitan} which examines the sensitivity of these hyperparameters.
Note that in our current instantiation, we use $\alpha_{V}=0$ so we only transfer policy networks across generations, since we found the value function was less informative.
We linearly anneal the coefficients for the supervised loss $\alpha_{V}$ and $\alpha_{\pi}$ from their original value to 0 over the course of the distillation phase.
This means we smoothly transition to a pure RL loss over the initial part of each new generation.

Our method is built using the PyTorch version of the \textsc{Brax}~\citep{brax} codebase at \url{https://github.com/google/brax/tree/main/brax}.
The codebase is open-sourced under the Apache 2.0 License.
The \textsc{Brax} environments are often subject to change, for full transparency, our evaluation is performed using the \textsc{0.10.0} version of the codebase.
We ran all our experiments on Nvidia Tesla V100 GPUs and used a single GPU for all experiments.
We note that the PPO baseline used in \cref{tab:main_results} is implemented in a different framework (JAX) to ours, which has some differences in network weight initialization.
The hyperparameters for the PPO baseline are tuned via grid-search on a reduced hyperparameter search space~\citep{brax}.
Since no hyperparameters were provided for the Hopper environment, we use the default in~\citet{brax}.  

For all experiments, we use $T_{\max} = 150M$, population size (number of parallel agents) $B = 8$ and $q = 12.5$ (percentage of the agents that are replaced at each \gls{PBT} iteration -- in this case, at each iteration, the single worst-performing agent is replaced).
For all environments except for Humanoid and Hopper, we use a fixed $t_{\mathrm{ready}} = 1M$.
To avoid excessive sensitivity to initialization, at the beginning of training for all \gls{PBT}-based methods (\gls{PB2}, \gls{PBT} and \our) we initialize with 24 agents and train for $t_{\mathrm{ready}}$ steps and choose the top-$B$ agents as the initializing population.
For the full \our, to trigger distillations and hence a new generation, we set a patience of 20 (i.e., if the evaluated return fails to improve after 20 consecutive $t_{\mathrm{ready}}$ steps, a new generation is started).
Since starting new generations can be desirable even if the training has not stalled, we introduce a second criterion to also start a new generation after 40M steps.
Thus, a new generation is started when either criterion is met (40M steps since last distillation, or 20 consecutive failures in improving the evaluated return).
For distillation at the start of every generation (all except initial), we begin distillation with 24 agents (4 suggested by \gls{BO} and the rest from random sampling, see App. \ref{appsubsec:new_archs} for details).
We then use successive halving to only distill $B$ of them using the full budget of 30M steps with the rest terminated early.

For the Humanoid and Hopper environments, we observed that PBT-style methods performed poorly across the board (See App. \ref{app:additional_experiments} for detailed results): in particular, on Hopper we notice that agents often learn a sub-optimal mode where it only learns to stand up (hence collecting the reward associating with simply surviving) but not to move.
On Humanoid, we find that agents often learn a mode where the humanoid does not use its knee joint -- in both cases, the agents seem to get stuck in stable but sub-optimal modes which use fewer degrees-of-freedom than they are capable of exploiting.
This behavior was ameliorated by linearly annealing the interval $t_{\mathrm{ready}}$ from 5M to 1M as a function of timesteps to not encourage myopic behavior at the start.
Since the increase in $t_{\mathrm{ready}}$ at the initial stage of training will lead to more exploratory behaviors, we increase the threshold before triggering a new generation at 60M for these two environments.

\section{Additional Experiments}
\label{app:additional_experiments}

\subsection{Scalability with Increased Training Budget and/or $B$}
\label{app:increased_timesteps}
\begin{wrapfigure}{R}{.35\textwidth}
\vspace{-12mm}
\centering{
\includegraphics[width=.95\linewidth]{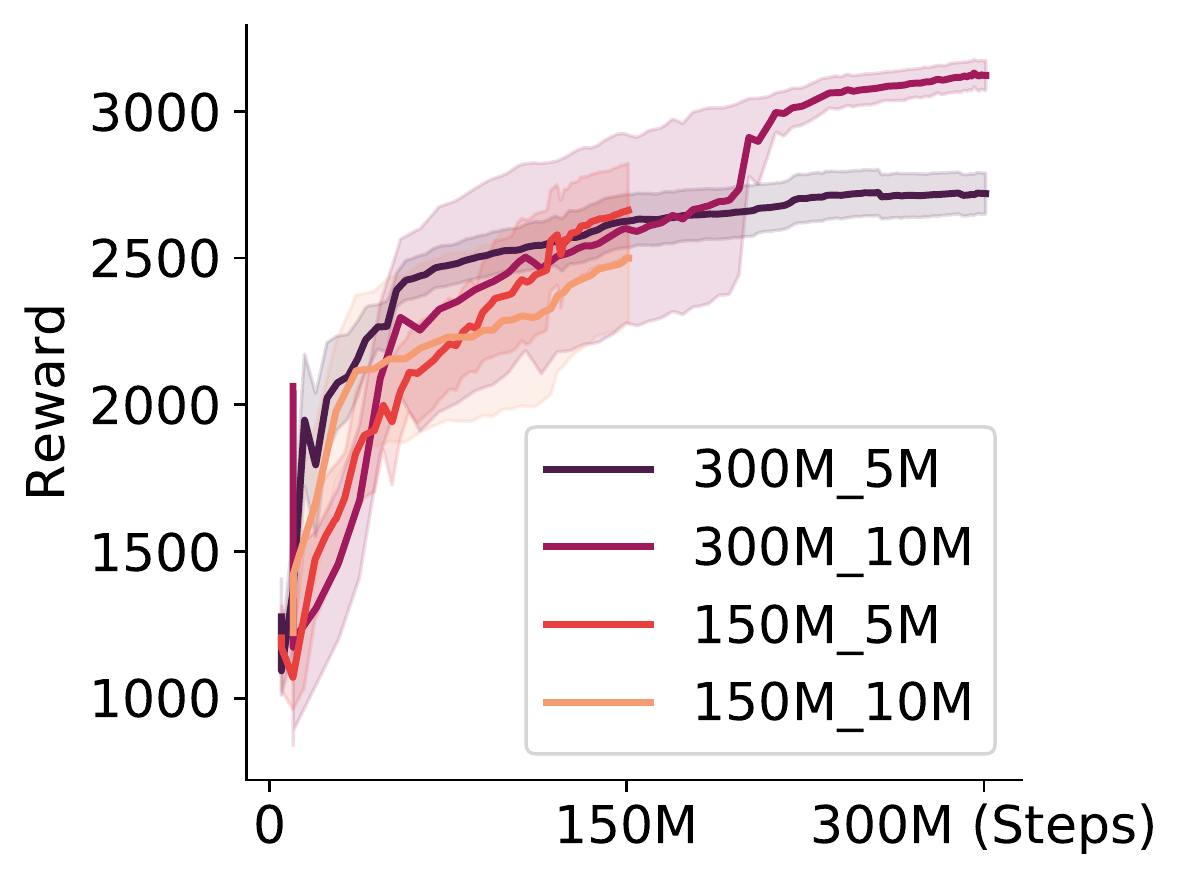}
}
\caption{\small{Evaluation of \our on the Hopper environment with $B=24$, and varying $t_{\max}$ and $t_{\mathrm{ready}}$. Results in the format ``\{$t_{\max}$\}\_\{$t_{\mathrm{ready}}$\}''.}}
\label{fig:hopper_longer_train}
\vspace{-8mm}
\end{wrapfigure}
For more complicated large-scale environments, \gls{PBT} is often used with a much larger population size than what we present in this paper ($B=8$).
In this section, we investigate whether \our benefits from increased parallelism by increasing the number of agents to $B=24$ and training for much longer.
In this instance, we use the Hopper environment as a testbed as it is amongst the most challenging in the current version of \textsc{Brax} and is particularly well suited to \gls{PBT}-style methods.
We show the results in Fig. \ref{fig:hopper_longer_train}: for a $t_{\max}$ of 150M, there is a small improvement over the $B=8$ results presented in the main text; whereas there is a significant benefit from jointly scaling up $B$, $t_{\max}$ and $t_{\mathrm{ready}}$ as exemplified by the \texttt{300M\_10M} result.

\subsection{Effects of Reduced Training Budget}
\label{app:decreased_timesteps}

\begin{figure}[t]
\centering
     \begin{subfigure}{1\linewidth}
          \includegraphics[width=\textwidth]{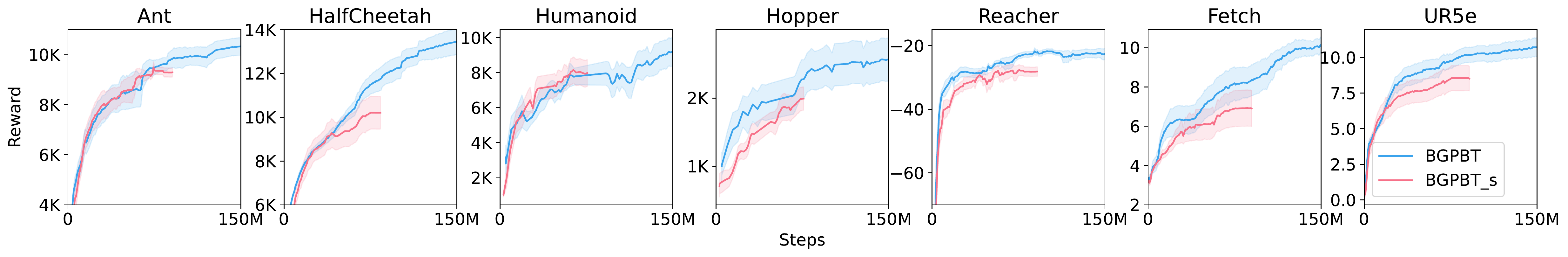}
     \end{subfigure}
     \vspace{-2mm}
    \vspace{-3.5mm}
\caption{
    \small{Mean evaluated return over the population with $\pm$1 \textsc{sem} (shaded) across 7 random seeds in all environments.
    }
}
\vspace{-4mm}
\label{fig:short_fig}
\end{figure}

We additionally conduct experiments on \our with the maximum timesteps roughly halved from the default $150$M used in the main experiments to remove the effect of the additional samples used during distillation.
We show the results in Fig. \ref{fig:short_fig}, where \gls{BG-PBT}\_{s} denotes \our run for roughly $75$M steps.
Compared to the training setup outlined in App. \ref{app:implementation_details}, to further reduce the training cost, we also reduce the number of initializing population to $12$, reduce the distillation timesteps to $20$M and allow for only one generation of distillation.
The results show that \our still performs well with results on par with or exceeding previous baselines using the full budget.

\subsection{Components of \our} 
\label{app:ablation}

\begin{figure}[h]
\vspace{-2mm}
\centering
     \begin{subfigure}{1\linewidth}
          \includegraphics[width=\textwidth]{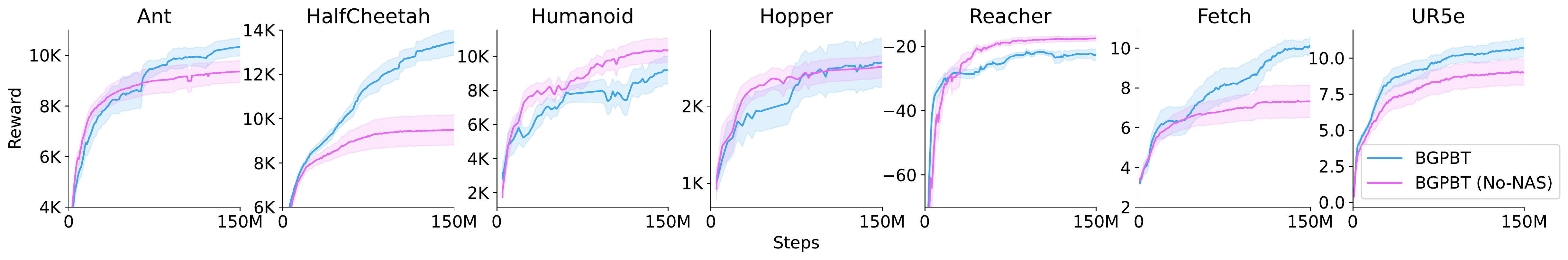}
     \end{subfigure}
     \vspace{-2mm}
    \vspace{-3.5mm}
\caption{
    \small{Comparison of full \our and \our without \gls{NAS} and distillation. Mean evaluated return over the population shown with $\pm$1 \textsc{sem} (shaded) across 7 random seeds in all environments.
    }
}
\label{fig:nonas_fig}
\end{figure}

We show the training trajectories of \our and the variant of \our without neural architecture search and distillation in Fig. \ref{fig:nonas_fig}. 
We also perform additional ablation studies on components of \our on Ant and HalfCheetah in Fig. \ref{fig:additional_ablation}.
The modifications to \our we consider are:
\begin{enumerate}
    \item \texttt{No NAS/TR} \gls{PBT} without \gls{TR}-based \gls{BO} or \gls{NAS}. This is identical to the \gls{PB2} baseline described in the main text.
    \item \texttt{No NAS} \our with \gls{TR}-based \gls{BO} in Sec. \ref{subsec:bo} but without \gls{NAS} or distillation.
    \item \texttt{Random Arch} \our with \gls{TR}-based \gls{BO} and distillation, but at the start of each generation, the architectures are selected randomly instead using \gls{BO}+\gls{RS} followed by the successive halving strategy described in App. \ref{appsubsec:new_archs}.
    \item \texttt{Static Arch} \our with \gls{TR}-based \gls{BO} and distillation, but without \gls{NAS}: all agents are started with the same default architectures for both the policy and value networks, and at the start of a new generation we distill across identical architectures.
\end{enumerate}
The results further demonstrate the benefit of introducing both \gls{TR}-based \gls{BO} and NAS to PBT-style methods in \our.
The results in Fig. \ref{fig:additional_ablation} also highlight the importance of having architecture diversity for distillation to be successful -- for both environments, removing the architecture variability (\texttt{Static Arch}) led to a significant drop in performance, which in some cases even under-performed the baseline without any distillation.
In contrast, simply initializing the new agents with random architectures performed surprisingly well (\texttt{Random Arch}).
This provides more even evidence that optimal architectures at different stages of training may vary, and thus should also vary dynamically through time.

\begin{figure}[b]
\centering
\vspace{-5mm}
\begin{subfigure}{0.5\textwidth}
\centering
     \includegraphics[width=\textwidth]{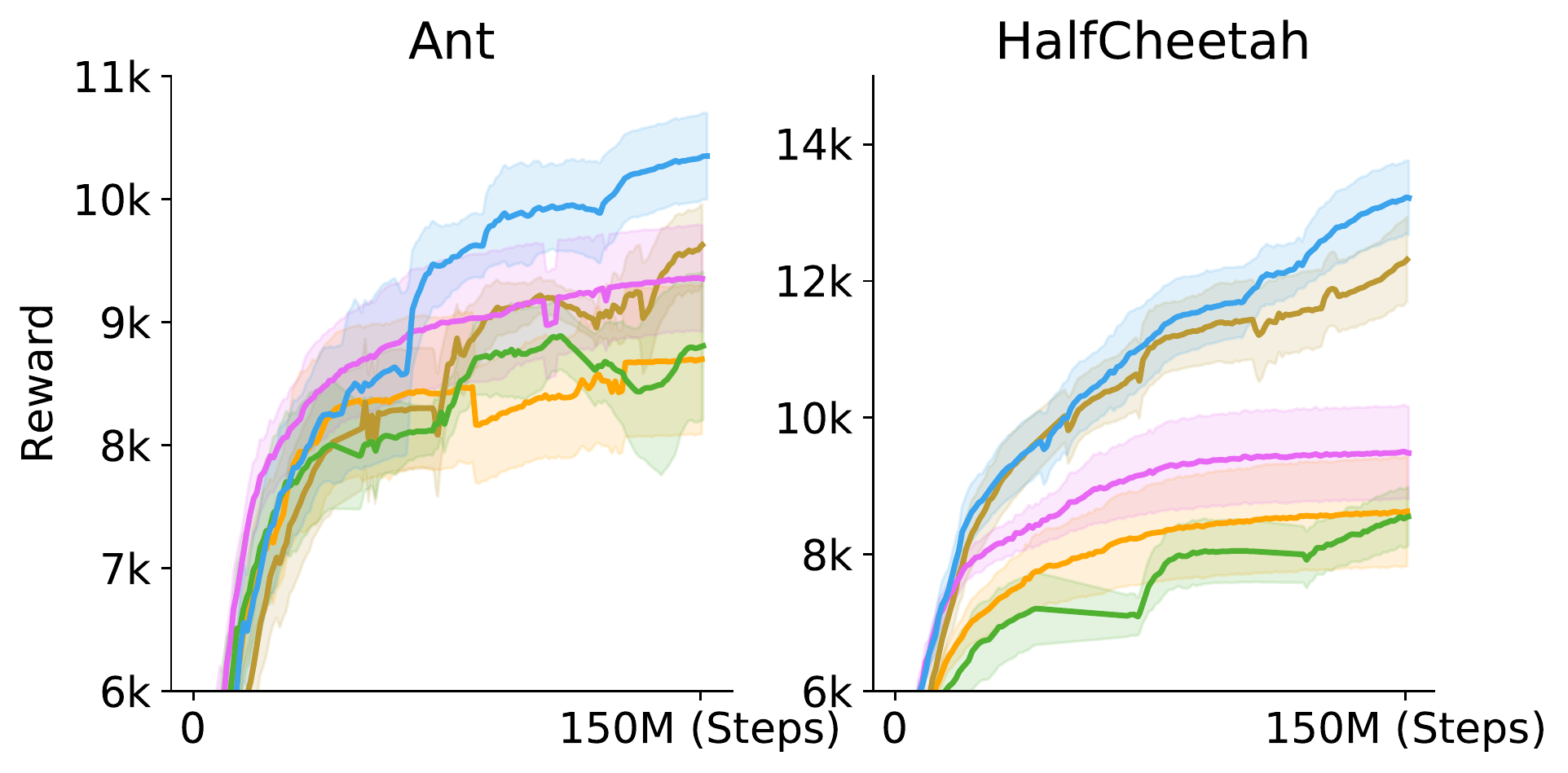}
\end{subfigure}
\begin{subfigure}{0.2\textwidth}
\centering
     \includegraphics[width=\textwidth]{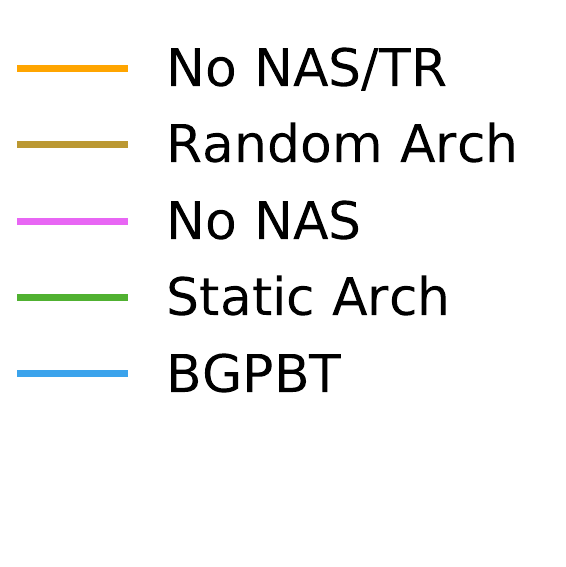}
\end{subfigure}
\caption{
    \small{Additional ablation studies on the Ant and HalfCheetah environments.}
}
\vspace{-5mm}
\label{fig:additional_ablation}
\end{figure}

\clearpage

\subsection{Hyperparameter and Architecture Schedules Learned on Additional Environments}
\label{app:more_schedule}
Supplementary to Fig. \ref{fig:ant_schedule}, we show the hyperparameter and architectures schedules learned by \our in Fig. \ref{fig:various_schedule} for additional environments where architecture search improves performance.
We see similar trends to the schedules learned for the Ant environment.

\begin{figure}[h]
\begin{subfigure}{1\textwidth}
\centering
     \includegraphics[width=\textwidth]{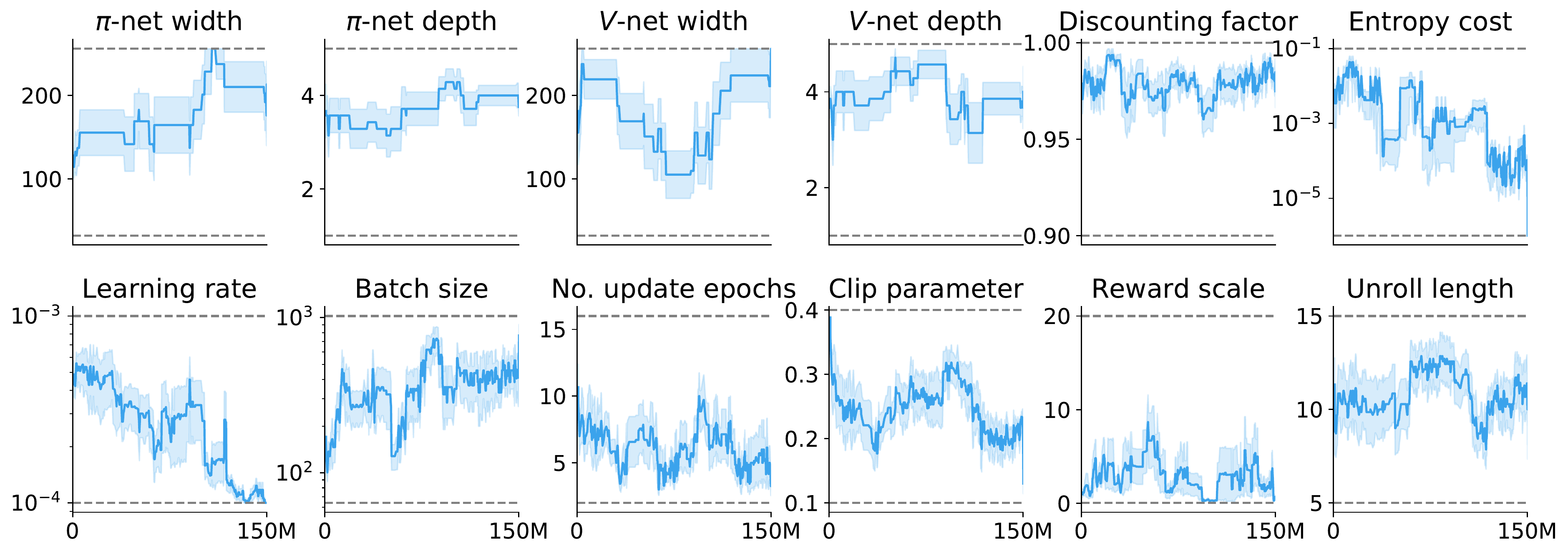}
     \caption{HalfCheetah}
\end{subfigure}
\begin{subfigure}{1\textwidth}
\centering
     \includegraphics[width=\textwidth]{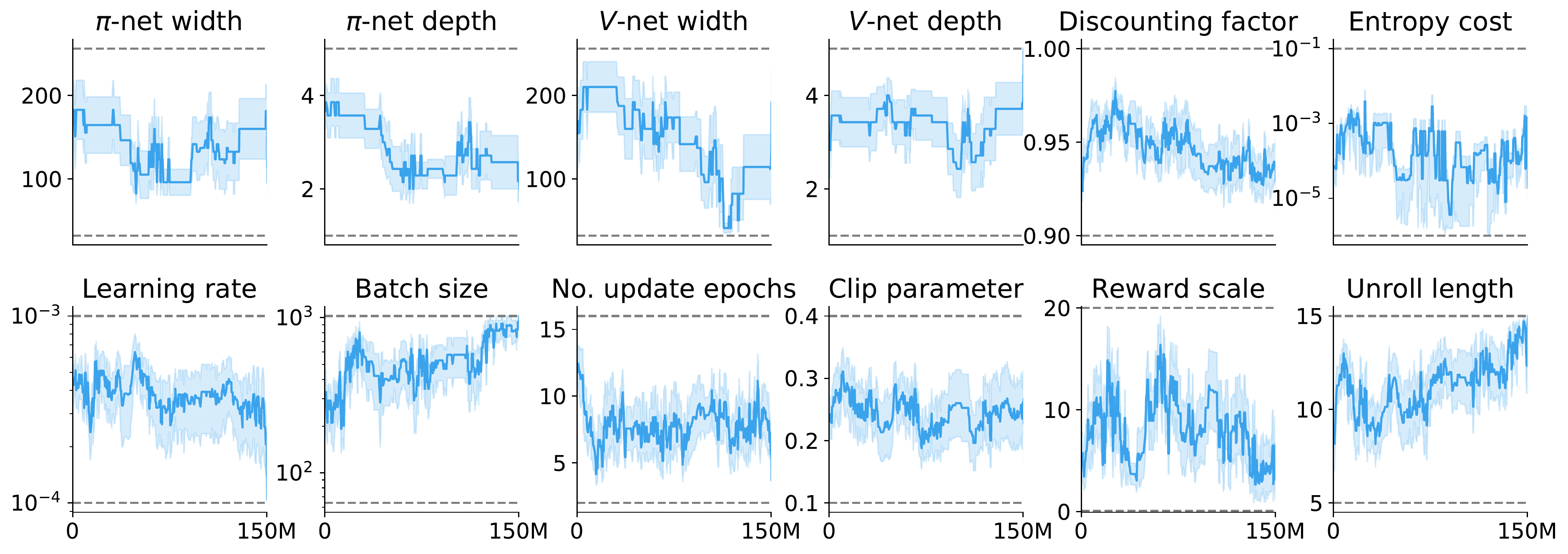}
     \caption{UR5e}
\end{subfigure}
\begin{subfigure}{1\textwidth}
\centering
     \includegraphics[width=\textwidth]{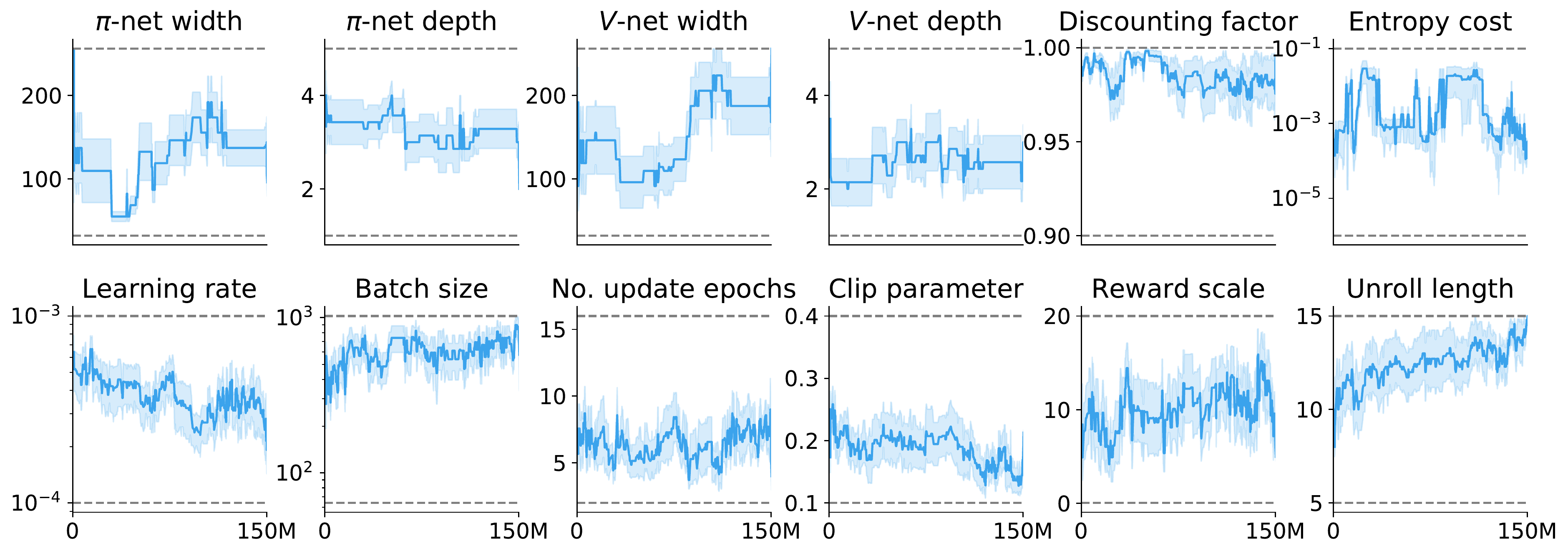}
     \caption{Fetch}
\end{subfigure}
\caption{
    \small{The hyperparameter and architecture schedule discovered by \our on various environments: we plot the hyperparameters of the best-performing agent in the population averaged across 7 seeds with $\pm$ 1 \textsc{sem} shaded. \textcolor{gray}{Gray dashed lines} denote the hyperparameter bounds.
    }
}
\label{fig:various_schedule}
\end{figure}

\clearpage

\subsection{Hopper and Humanoid Results with Constant $t_{\mathrm{ready}} = 1M$}
\label{app:hopper_constant}

For our main evaluation, we linearly anneal $t_{\mathrm{ready}}$ through time.
We show why this is necessary in Table~\ref{tab:hopper_humanoid_constant_tready} by evaluating the PBT-style methods with a constant $t_{\mathrm{ready}} = 1\textrm{M}$ as with the other environments.
We observe a considerable decrease in evaluated return compared to the results in \cref{tab:main_results}.

\begin{table}[h]
    \centering
        \caption{{Mean evaluated return $\pm 1$\textsc{sem} across 7 seeds shown for Humanoid and Hopper environments with constant $t_{\mathrm{ready}}=1\textrm{M}$.
        }
        }
        \begin{footnotesize}
    \adjustbox{max width=0.98\columnwidth}{
    \begin{tabular}{cccc}
    \toprule
     Method & PBT & PB2 & BG-PBT\\
     Search space & $\mathcal{Z}$ & $\mathcal{Z}$ & $\mathcal{J}$ \\
    \midrule
    Humanoid   & $\mathbf{7498_{\pm 666}}$ & $\mathbf{7667_{\pm 1000}}$ & $\mathbf{7949_{\pm 876}}$ \\
    Hopper  &  $1667_{\pm 222}$ & $1253_{\pm 77}$ &  $\mathbf{2257_{\pm 290}}$\\
    \bottomrule 
    \end{tabular}}
    \label{tab:hopper_humanoid_constant_tready}
    \end{footnotesize}
\end{table}

\subsection{Ablations on Explicit Treatment of Ordinal Variables}
\label{app:ablation_ordinal}

For \our, we extended Casmopolitan \citep{casmopolitan} by further accommodating for ordinal variables (i.e., discrete variables with ordering, such as integer variables for the width of a multi-layer perceptron). In this section, we show the empirical benefits of this over treating them as categorical variables (the default in \cite{casmopolitan}).

We consider 3 Brax environments (Fig. \ref{fig:categorical}), where we further compare the full \our with a variant of \our where we use the original Casmopolitan by treating the ordinal variables ($\log_2$-batch size, widths and depths of the value and policy MLPs, unroll length, number of updates per epoch) as categorical variables. It is clear that the ordinal treatment in our case results in better performance in all 3 environments.

\begin{figure}[h]
\centering
\begin{subfigure}{0.85\textwidth}
\centering
     \includegraphics[width=\textwidth]{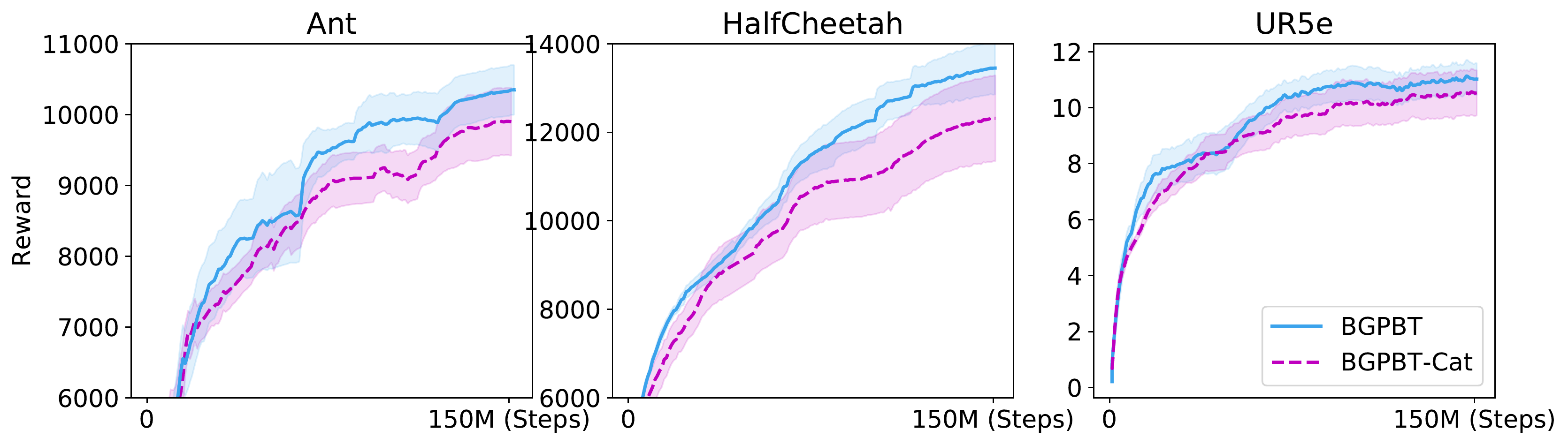}
\end{subfigure}
\caption{
    \small{Comparison of \our with the variant of \our that uses original Casmpolitan (\our-Cat) that treats ordinal variables as categorical.}
}
\label{fig:categorical}
\end{figure}

\subsection{Comparison Against Random Search with Manually Defined Learning Rate Schedule}
\label{app:comp_rs}

We compare against random search that uses an equivalent amount of compute resources, but with a manually defined learning rate schedule, in Table \ref{tab:rs_anneal}.
Specifically, we use random search for all hyperparameters in either the joint search space ($\mathcal{J}$) or the hyperparameter search space ($\mathcal{Z}$) defined in Table \ref{tab:main_results}, with the exception that instead of using flat learning rates, we search for an \emph{initial} learning rate which is cosine annealed to $10^{-8}$ by the end of the training.
We find that in Ant and HalfCheetah where the \our discovered schedules are similar to the manual cosine schedules, \texttt{RS-Anneal} significantly outperforms regular RS, but when the discovered schedules deviate from the manual design in the case of UR5e, we find the margin of improvement to be much smaller.
This shows that there is unlikely to be an optimal manual schedule for all environments, further demonstrating the desirable flexibility of \our in adapting to different environments.
Furthermore, in all environments, \our still outperforms RS, with or without manual learning rate scheduling.
This is particularly notable as previous \gls{PBT}-style methods were often known to under-perform random search, especially with small population sizes.

\begin{table}[tb]
    \centering
        \caption{Comparison against RS and RS with cosine annealing (RS-Anneal). The results for RS and \our are lifted from Table \ref{tab:main_results}.
        }
        \vspace{-3mm}
        \begin{footnotesize}
    \adjustbox{max width=0.98\columnwidth}{
    \begin{tabular}{cccccc}
    \toprule
     Method & RS & RS & RS-Anneal & RS-Anneal & BG-PBT\\
     Search space & $\mathcal{Z}$ & $\mathcal{J}$ & $\mathcal{Z}$ & $\mathcal{J}$  & $\mathcal{J}$\\
    \midrule
    Ant & $6780_{\pm 317}$ & $4781_{\pm 515}$ & $9640_{\pm 79}$ & $9536_{\pm 257}$ & $\mathbf{10349_{\pm 326}}$ \\
    HalfCheetah & $9502_{\pm 76}$ & $10340_{\pm 329}$ & $9672_{\pm 157}$ & $\mathbf{13071_{\pm 360}}$ & $\mathbf{13450_{\pm 551}}$ \\
    UR5e &  $5.3_{\pm 0.4}$ &  $6.9_{\pm 0.4}$ & $7.7_{\pm 0.3}$ & $7.8_{\pm 0.4}$ &$\mathbf{10.7_{\pm 0.6}}$\\
    \bottomrule 
    \end{tabular}}
    \label{tab:rs_anneal}
    \end{footnotesize}
\end{table}

%% file: main.bbl
\begin{thebibliography}{}

\bibitem[Andrychowicz et~al., 2021]{andrychowicz2021what}
Andrychowicz, M., Raichuk, A., Sta{\'n}czyk, P., Orsini, M., Girgin, S.,
  Marinier, R., Hussenot, L., Geist, M., Pietquin, O., Michalski, M., Gelly,
  S., and Bachem, O. (2021).
\newblock What matters for on-policy deep actor-critic methods? a large-scale
  study.
\newblock In {\em International Conference on Learning Representations}.

\bibitem[Badia et~al., 2020]{agent57}
Badia, A.~P., Piot, B., Kapturowski, S., Sprechmann, P., Vitvitskyi, A., Guo,
  Z.~D., and Blundell, C. (2020).
\newblock Agent57: Outperforming the atari human benchmark.
\newblock In {\em Proceedings of the 37th International Conference on Machine
  Learning, {ICML}, 13-18 July, Virtual Event}, volume 119 of {\em Proceedings
  of Machine Learning Research}, pages 507--517. {PMLR}.

\bibitem[Ball et~al., 2020]{rp1}
Ball, P., Parker-Holder, J., Pacchiano, A., Choromanski, K., and Roberts, S.
  (2020).
\newblock Ready policy one: World building through active learning.
\newblock In {\em Proceedings of the 37th International Conference on Machine
  Learning, {ICML}}.

\bibitem[Bibby, 1974]{bibby1974axiomatisations}
Bibby, J. (1974).
\newblock Axiomatisations of the average and a further generalisation of
  monotonic sequences.
\newblock {\em Glasgow Mathematical Journal}, 15(1):63--65.

\bibitem[Bogunovic et~al., 2016]{bogunovic2016time}
Bogunovic, I., Scarlett, J., and Cevher, V. (2016).
\newblock Time-varying gaussian process bandit optimization.
\newblock In {\em Artificial Intelligence and Statistics}, pages 314--323.

\bibitem[Brochu et~al., 2010]{Brochu_2010Tutorial}
Brochu, E., Cora, V.~M., and de~Freitas, N. (2010).
\newblock A tutorial on bayesian optimization of expensive cost functions, with
  application to active user modeling and hierarchical reinforcement learning.
\newblock {\em ArXiv}, abs/1012.2599.

\bibitem[Chen et~al., 2018]{bo_alphago}
Chen, Y., Huang, A., Wang, Z., Antonoglou, I., Schrittwieser, J., Silver, D.,
  and de~Freitas, N. (2018).
\newblock {B}ayesian optimization in {AlphaGo}.
\newblock {\em CoRR}, abs/1812.06855.

\bibitem[Cobbe et~al., 2020]{procgen}
Cobbe, K., Hesse, C., Hilton, J., and Schulman, J. (2020).
\newblock Leveraging procedural generation to benchmark reinforcement learning.

\bibitem[Cobbe et~al., 2019]{coinrun}
Cobbe, K., Klimov, O., Hesse, C., Kim, T., and Schulman, J. (2019).
\newblock Quantifying generalization in reinforcement learning.
\newblock In Chaudhuri, K. and Salakhutdinov, R., editors, {\em Proceedings of
  the 36th International Conference on Machine Learning}, volume~97 of {\em
  Proceedings of Machine Learning Research}, pages 1282--1289. PMLR.

\bibitem[Czarnecki et~al., 2018]{czarnecki2018mixmatch}
Czarnecki, W., Jayakumar, S., Jaderberg, M., Hasenclever, L., Teh, Y.~W.,
  Heess, N., Osindero, S., and Pascanu, R. (2018).
\newblock Mix \& match agent curricula for reinforcement learning.
\newblock In {\em Proceedings of the 35th International Conference on Machine
  Learning}, volume~80 of {\em Proceedings of Machine Learning Research}, pages
  1087--1095. PMLR.

\bibitem[Dalibard and Jaderberg, 2021]{fire_pbt}
Dalibard, V. and Jaderberg, M. (2021).
\newblock Faster improvement rate population based training.
\newblock {\em CoRR}, abs/2109.13800.

\bibitem[Engstrom et~al., 2020]{Engstrom2020Implementation}
Engstrom, L., Ilyas, A., Santurkar, S., Tsipras, D., Janoos, F., Rudolph, L.,
  and Madry, A. (2020).
\newblock Implementation matters in deep {RL:} {A} case study on {PPO} and
  {TRPO}.
\newblock In {\em 8th International Conference on Learning Representations,
  {ICLR}, Addis Ababa, Ethiopia, April 26-30}. OpenReview.net.

\bibitem[Eriksson et~al., 2019]{eriksson2019scalable}
Eriksson, D., Pearce, M., Gardner, J., Turner, R.~D., and Poloczek, M. (2019).
\newblock Scalable global optimization via local bayesian optimization.
\newblock {\em Advances in Neural Information Processing Systems}, 32.

\bibitem[Flennerhag et~al., 2021]{flennerhag2022bootstrapped}
Flennerhag, S., Schroecker, Y., Zahavy, T., van Hasselt, H., Silver, D., and
  Singh, S. (2021).
\newblock Bootstrapped meta-learning.
\newblock In {\em arxiv}.

\bibitem[Franke et~al., 2021]{franke2020sample}
Franke, J.~K., Koehler, G., Biedenkapp, A., and Hutter, F. (2021).
\newblock Sample-efficient automated deep reinforcement learning.
\newblock In {\em International Conference on Learning Representations}.

\bibitem[Frazier, 2018]{frazier2018tutorial}
Frazier, P.~I. (2018).
\newblock A tutorial on {B}ayesian optimization.
\newblock {\em arXiv preprint arXiv:1807.02811}.

\bibitem[Freeman et~al., 2021]{brax}
Freeman, C.~D., Frey, E., Raichuk, A., Girgin, S., Mordatch, I., and Bachem, O.
  (2021).
\newblock Brax - a differentiable physics engine for large scale rigid body
  simulation.

\bibitem[Fu et~al., 2020]{fu2020autoagentdistiller}
Fu, Y., Yu, Z., Zhang, Y., and Lin, Y. (2020).
\newblock Auto-agent-distiller: Towards efficient deep reinforcement learning
  agents via neural architecture search.
\newblock {\em arXiv preprint arXiv:2012.13091}.

\bibitem[Furuta et~al., 2021]{furuta2021pic}
Furuta, H., Matsushima, T., Kozuno, T., Matsuo, Y., Levine, S., Nachum, O., and
  Gu, S.~S. (2021).
\newblock Policy information capacity: Information-theoretic measure for task
  complexity in deep reinforcement learning.
\newblock In {\em International Conference on Machine Learning}.

\bibitem[Henderson et~al., 2018]{deeprlmatters}
Henderson, P., Islam, R., Bachman, P., Pineau, J., Precup, D., and Meger, D.
  (2018).
\newblock Deep reinforcement learning that matters.
\newblock In {\em Proceedings of the Thirty-Second {AAAI} Conference on
  Artificial Intelligence, (AAAI)}, pages 3207--3214. {AAAI} Press.

\bibitem[Igl et~al., 2021]{igl2021transient}
Igl, M., Farquhar, G., Luketina, J., Boehmer, W., and Whiteson, S. (2021).
\newblock Transient non-stationarity and generalisation in deep reinforcement
  learning.
\newblock In {\em International Conference on Learning Representations}.

\bibitem[Izquierdo et~al., 2021]{izquierdo2021bag}
Izquierdo, S., Guerrero-Viu, J., Hauns, S., Miotto, G., Schrodi, S.,
  Biedenkapp, A., Elsken, T., Deng, D., Lindauer, M., and Hutter, F. (2021).
\newblock Bag of baselines for multi-objective joint neural architecture search
  and hyperparameter optimization.
\newblock In {\em 8th ICML Workshop on Automated Machine Learning (AutoML)}.

\bibitem[Jaderberg et~al., 2019]{capturetheflag}
Jaderberg, M., Czarnecki, W.~M., Dunning, I., Marris, L., Lever, G.,
  Casta{\~n}eda, A.~G., Beattie, C., Rabinowitz, N.~C., Morcos, A.~S.,
  Ruderman, A., Sonnerat, N., Green, T., Deason, L., Leibo, J.~Z., Silver, D.,
  Hassabis, D., Kavukcuoglu, K., and Graepel, T. (2019).
\newblock Human-level performance in 3d multiplayer games with population-based
  reinforcement learning.
\newblock {\em Science}, 364(6443):859--865.

\bibitem[Jaderberg et~al., 2017]{PBT}
Jaderberg, M., Dalibard, V., Osindero, S., Czarnecki, W.~M., Donahue, J.,
  Razavi, A., Vinyals, O., Green, T., Dunning, I., Simonyan, K., Fernando, C.,
  and Kavukcuoglu, K. (2017).
\newblock Population based training of neural networks.
\newblock {\em CoRR}, abs/1711.09846.

\bibitem[Jamieson and Talwalkar, 2016]{pmlr-v51-jamieson16}
Jamieson, K. and Talwalkar, A. (2016).
\newblock Non-stochastic best arm identification and hyperparameter
  optimization.
\newblock In Gretton, A. and Robert, C.~C., editors, {\em Proceedings of the
  19th International Conference on Artificial Intelligence and Statistics},
  volume~51 of {\em Proceedings of Machine Learning Research}, pages 240--248,
  Cadiz, Spain. PMLR.

\bibitem[Kalashnikov et~al., 2018]{qtopt}
Kalashnikov, D., Irpan, A., Pastor, P., Ibarz, J., Herzog, A., Jang, E.,
  Quillen, D., Holly, E., Kalakrishnan, M., Vanhoucke, V., and Levine, S.
  (2018).
\newblock Scalable deep reinforcement learning for vision-based robotic
  manipulation.
\newblock In Billard, A., Dragan, A., Peters, J., and Morimoto, J., editors,
  {\em Proceedings of The 2nd Conference on Robot Learning}, volume~87 of {\em
  Proceedings of Machine Learning Research}, pages 651--673. PMLR.

\bibitem[Kandasamy et~al., 2018]{kandasamy2018neural}
Kandasamy, K., Neiswanger, W., Schneider, J., Poczos, B., and Xing, E.~P.
  (2018).
\newblock Neural architecture search with bayesian optimisation and optimal
  transport.
\newblock {\em Advances in neural information processing systems}, 31.

\bibitem[Karnin et~al., 2013]{pmlr-v28-karnin13}
Karnin, Z., Koren, T., and Somekh, O. (2013).
\newblock Almost optimal exploration in multi-armed bandits.
\newblock In Dasgupta, S. and McAllester, D., editors, {\em Proceedings of the
  30th International Conference on Machine Learning}, volume~28 of {\em
  Proceedings of Machine Learning Research}, pages 1238--1246, Atlanta,
  Georgia, USA. PMLR.

\bibitem[Lindauer et~al., 2022]{lindauer2022smac3}
Lindauer, M., Eggensperger, K., Feurer, M., Biedenkapp, A., Deng, D.,
  Benjamins, C., Ruhkopf, T., Sass, R., and Hutter, F. (2022).
\newblock Smac3: A versatile bayesian optimization package for hyperparameter
  optimization.
\newblock {\em Journal of Machine Learning Research}, 23(54):1--9.

\bibitem[Liu et~al., 2019]{liu2018darts}
Liu, H., Simonyan, K., and Yang, Y. (2019).
\newblock {DARTS}: Differentiable architecture search.
\newblock In {\em International Conference on Learning Representations}.

\bibitem[Liu et~al., 2021]{pbtfootball}
Liu, S., Lever, G., Wang, Z., Merel, J., Eslami, S. M.~A., Hennes, D.,
  Czarnecki, W.~M., Tassa, Y., Omidshafiei, S., Abdolmaleki, A., Siegel, N.~Y.,
  Hasenclever, L., Marris, L., Tunyasuvunakool, S., Song, H.~F., Wulfmeier, M.,
  Muller, P., Haarnoja, T., Tracey, B.~D., Tuyls, K., Graepel, T., and Heess,
  N. (2021).
\newblock From motor control to team play in simulated humanoid football.
\newblock {\em CoRR}, abs/2105.12196.

\bibitem[Miao et~al., 2021]{rl_darts}
Miao, Y., Song, X., Peng, D., Yue, S., Brevdo, E., and Faust, A. (2021).
\newblock {RL-DARTS:} differentiable architecture search for reinforcement
  learning.
\newblock {\em CoRR}, abs/2106.02229.

\bibitem[Mnih et~al., 2013]{mnih-atari-2013}
Mnih, V., Kavukcuoglu, K., Silver, D., Graves, A., Antonoglou, I., Wierstra,
  D., and Riedmiller, M.~A. (2013).
\newblock Playing atari with deep reinforcement learning.
\newblock {\em CoRR}, abs/1312.5602.

\bibitem[Moskovitz et~al., 2021]{top}
Moskovitz, T., Parker{-}Holder, J., Pacchiano, A., and Arbel, M. (2021).
\newblock Deep reinforcement learning with dynamic optimism.
\newblock In {\em Advances in Neural Information Processing Systems}.

\bibitem[Nguyen et~al., 2021a]{nguyen2021optimal}
Nguyen, V., Le, T., Yamada, M., and Osborne, M.~A. (2021a).
\newblock Optimal transport kernels for sequential and parallel neural
  architecture search.
\newblock In {\em International Conference on Machine Learning}, pages
  8084--8095. PMLR.

\bibitem[Nguyen et~al., 2020]{gpbo_tvo}
Nguyen, V., Masrani, V., Brekelmans, R., Osborne, M., and Wood, F. (2020).
\newblock Gaussian process bandit optimization of the thermodynamic variational
  objective.
\newblock {\em Advances in Neural Information Processing Systems},
  33:5764--5775.

\bibitem[Nguyen et~al., 2021b]{nguyen2021deep}
Nguyen, V., Orbell, S., Lennon, D.~T., Moon, H., Vigneau, F., Camenzind, L.~C.,
  Yu, L., Zumb{\"u}hl, D.~M., Briggs, G. A.~D., Osborne, M.~A., et~al. (2021b).
\newblock Deep reinforcement learning for efficient measurement of quantum
  devices.
\newblock {\em npj Quantum Information}, 7(1):1--9.

\bibitem[Ota et~al., 2021]{ota2021training}
Ota, K., Jha, D.~K., and Kanezaki, A. (2021).
\newblock Training larger networks for deep reinforcement learning.

\bibitem[Park et~al., 2019]{networkwidthlearningrate}
Park, D., Sohl-Dickstein, J., Le, Q., and Smith, S. (2019).
\newblock The effect of network width on stochastic gradient descent and
  generalization: an empirical study.
\newblock In Chaudhuri, K. and Salakhutdinov, R., editors, {\em Proceedings of
  the 36th International Conference on Machine Learning}, volume~97 of {\em
  Proceedings of Machine Learning Research}, pages 5042--5051. PMLR.

\bibitem[Parker-Holder et~al., 2021]{parkerholder2021tuning}
Parker-Holder, J., Nguyen, V., Desai, S., and Roberts, S. (2021).
\newblock Tuning mixed input hyperparameters on the fly for efficient
  population based auto{RL}.
\newblock In Beygelzimer, A., Dauphin, Y., Liang, P., and Vaughan, J.~W.,
  editors, {\em Advances in Neural Information Processing Systems}.

\bibitem[Parker-Holder et~al., 2020]{pb2}
Parker-Holder, J., Nguyen, V., and Roberts, S.~J. (2020).
\newblock Provably efficient online hyperparameter optimization with
  population-based bandits.
\newblock In Larochelle, H., Ranzato, M., Hadsell, R., Balcan, M.~F., and Lin,
  H., editors, {\em Advances in Neural Information Processing Systems},
  volume~33, pages 17200--17211. Curran Associates, Inc.

\bibitem[Parker{-}Holder et~al., 2020]{dvd}
Parker{-}Holder, J., Pacchiano, A., Choromanski, K., and Roberts, S. (2020).
\newblock Effective diversity in population-based reinforcement learning.
\newblock In {\em Advances in Neural Information Processing Systems 33}.

\bibitem[Parker{-}Holder et~al., 2022]{autorl_survey}
Parker{-}Holder, J., Rajan, R., Song, X., Biedenkapp, A., Miao, Y., Eimer, T.,
  Zhang, B., Nguyen, V., Calandra, R., Faust, A., Hutter, F., and Lindauer, M.
  (2022).
\newblock Automated reinforcement learning (autorl): {A} survey and open
  problems.
\newblock {\em Journal of Artificial Intelligence Research}, abs/2201.03916.

\bibitem[Paul et~al., 2016]{DBLP:journals/corr/PaulCOW16}
Paul, S., Ciosek, K., Osborne, M.~A., and Whiteson, S. (2016).
\newblock Alternating optimisation and quadrature for robust reinforcement
  learning.
\newblock {\em CoRR}, abs/1605.07496.

\bibitem[Paul et~al., 2019]{paul2019fast}
Paul, S., Kurin, V., and Whiteson, S. (2019).
\newblock Fast efficient hyperparameter tuning for policy gradient methods.
\newblock In {\em Advances in Neural Information Processing Systems},
  volume~32.

\bibitem[Rasmussen and Williams, 2006]{RasmussenGP}
Rasmussen, C.~E. and Williams, C. K.~I. (2006).
\newblock {\em Gaussian Processes for Machine Learning}.
\newblock MIT Press.

\bibitem[Ru et~al., 2021]{ru2020interpretable}
Ru, B., Wan, X., Dong, X., and Osborne, M. (2021).
\newblock Interpretable neural architecture search via bayesian optimisation
  with weisfeiler-lehman kernels.
\newblock {\em International Conference on Learning Representations}.

\bibitem[Schulman et~al., 2015]{trpo}
Schulman, J., Levine, S., Abbeel, P., Jordan, M., and Moritz, P. (2015).
\newblock Trust region policy optimization.
\newblock In Bach, F. and Blei, D., editors, {\em Proceedings of the 32nd
  International Conference on Machine Learning}, volume~37 of {\em Proceedings
  of Machine Learning Research}, pages 1889--1897, Lille, France. PMLR.

\bibitem[Schulman et~al., 2017]{ppo}
Schulman, J., Wolski, F., Dhariwal, P., Radford, A., and Klimov, O. (2017).
\newblock Proximal policy optimization algorithms.

\bibitem[Shahriari et~al., 2016]{Shahriari_2016Taking}
Shahriari, B., Swersky, K., Wang, Z., Adams, R.~P., and de~Freitas, N. (2016).
\newblock Taking the human out of the loop: A review of {B}ayesian
  optimization.
\newblock {\em Proceedings of the IEEE}, 104(1):148--175.

\bibitem[Silver et~al., 2017]{alphago}
Silver, D., Schrittwieser, J., Simonyan, K., Antonoglou, I., Huang, A., Guez,
  A., Hubert, T., baker, L., Lai, M., Bolton, A., Chen, Y., Lillicrap, T.~P.,
  Hui, F., Sifre, L., van~den Driessche, G., Graepel, T., and Hassabis, D.
  (2017).
\newblock Mastering the game of go without human knowledge.
\newblock {\em Nature}, 550:354--359.

\bibitem[Silver et~al., 2021]{rewardenough}
Silver, D., Singh, S., Precup, D., and Sutton, R.~S. (2021).
\newblock Reward is enough.
\newblock {\em Artificial Intelligence}, 299:103535.

\bibitem[Srinivas et~al., 2010]{Srinivas_2010Gaussian}
Srinivas, N., Krause, A., Kakade, S., and Seeger, M. (2010).
\newblock Gaussian process optimization in the bandit setting: No regret and
  experimental design.
\newblock In {\em Proceedings of the 27th International Conference on Machine
  Learning}, pages 1015--1022.

\bibitem[Stooke et~al., 2021]{openendedlearningteam2021openended}
Stooke, A., Mahajan, A., Barros, C., Deck, C., Bauer, J., Sygnowski, J.,
  Trebacz, M., Jaderberg, M., Mathieu, M., McAleese, N., Bradley{-}Schmieg, N.,
  Wong, N., Porcel, N., Raileanu, R., Hughes{-}Fitt, S., Dalibard, V., and
  Czarnecki, W.~M. (2021).
\newblock Open-ended learning leads to generally capable agents.
\newblock {\em arXiv preprint arXiv:2107.12808}.

\bibitem[Sutton and Barto, 2018]{Sutton1998}
Sutton, R.~S. and Barto, A.~G. (2018).
\newblock {\em Reinforcement Learning: An Introduction}.
\newblock The MIT Press, second edition.

\bibitem[Tang and Choromanski, 2020]{online_hpo_evolutionary_arvix_20}
Tang, Y. and Choromanski, K. (2020).
\newblock Online hyper-parameter tuning in off-policy learning via evolutionary
  strategies.
\newblock {\em CoRR}, abs/2006.07554.

\bibitem[Vinyals et~al., 2019]{alphastar}
Vinyals, O., Babuschkin, I., Czarnecki, W.~M., Mathieu, M., Dudzik, A., Chung,
  J., Choi, D.~H., Powell, R., Ewalds, T., Georgiev, P., Oh, J., Horgan, D.,
  Kroiss, M., Danihelka, I., Huang, A., Sifre, L., Cai, T., Agapiou, J.~P.,
  Jaderberg, M., Vezhnevets, A.~S., Leblond, R., Pohlen, T., Dalibard, V.,
  Budden, D., Sulsky, Y., Molloy, J., Paine, T.~L., Gulcehre, C., Wang, Z.,
  Pfaff, T., Wu, Y., Ring, R., Yogatama, D., W{\"u}nsch, D., McKinney, K.,
  Smith, O., Schaul, T., Lillicrap, T.~P., Kavukcuoglu, K., Hassabis, D., Apps,
  C., and Silver, D. (2019).
\newblock Grandmaster level in starcraft ii using multi-agent reinforcement
  learning.
\newblock {\em Nature}, pages 1--5.

\bibitem[Wan et~al., 2021]{casmopolitan}
Wan, X., Nguyen, V., Ha, H., Ru, B., Lu, C., and Osborne, M.~A. (2021).
\newblock Think global and act local: Bayesian optimisation over
  high-dimensional categorical and mixed search spaces.
\newblock In Meila, M. and Zhang, T., editors, {\em Proceedings of the 38th
  International Conference on Machine Learning}, volume 139 of {\em Proceedings
  of Machine Learning Research}, pages 10663--10674. PMLR.

\bibitem[Wan et~al., 2022]{wan2022approximate}
Wan, X., Ru, B., Esparan{\c{c}}a, P.~M., and Carlucci, F.~M. (2022).
\newblock Approximate neural architecture search via operation distribution
  learning.
\newblock In {\em Proceedings of the IEEE/CVF Winter Conference on Applications
  of Computer Vision}, pages 2377--2386.

\bibitem[White et~al., 2021]{white2019bananas}
White, C., Neiswanger, W., and Savani, Y. (2021).
\newblock Bananas: Bayesian optimization with neural architectures for neural
  architecture search.
\newblock {\em AAAI}, 1(2):4.

\bibitem[Xu et~al., 2022]{xu2022accelerated}
Xu, J., Macklin, M., Makoviychuk, V., Narang, Y., Garg, A., Ramos, F., and
  Matusik, W. (2022).
\newblock Accelerated policy learning with parallel differentiable simulation.
\newblock In {\em International Conference on Learning Representations}.

\bibitem[Xu et~al., 2018]{metagradients}
Xu, Z., van Hasselt, H., and Silver, D. (2018).
\newblock Meta-gradient reinforcement learning.
\newblock In {\em Advances in Neural Information Processing Systems 31:
  NeurIPS, December 3-8, Montr{\'{e}}al, Canada}, pages 2402--2413.

\bibitem[Yuan, 1999]{yuan2000review}
Yuan, Y.-x. (1999).
\newblock A review of trust region algorithms for optimization.
\newblock {\em ICM99: Proceedings of the Fourth International Congress on
  Industrial and Applied Mathematics}.

\bibitem[Zahavy et~al., 2020]{stac}
Zahavy, T., Xu, Z., Veeriah, V., Hessel, M., Oh, J., van Hasselt, H., Silver,
  D., and Singh, S. (2020).
\newblock A self-tuning actor-critic algorithm.
\newblock In {\em Advances in Neural Information Processing Systems}.

\bibitem[Zhang et~al., 2021]{pbtbt}
Zhang, B., Rajan, R., Pineda, L., Lambert, N., Biedenkapp, A., Chua, K.,
  Hutter, F., and Calandra, R. (2021).
\newblock On the importance of hyperparameter optimization for model-based
  reinforcement learning.
\newblock In {\em Proceedings of The 24th International Conference on
  Artificial Intelligence and Statistics}.

\end{thebibliography}
